\documentclass{article}
\usepackage{fullpage}

\usepackage{graphicx}
\usepackage{amssymb}
\usepackage{amstext}
\usepackage{epsfig}

\usepackage{array,multirow}

\usepackage{latexsym}
\usepackage{amsmath}
\usepackage{bm} 
\usepackage{bbm}

\usepackage{amsthm}
\usepackage{amsmath}
\usepackage{amsfonts}
\usepackage{amssymb,amstext}
\usepackage{appendix}
\usepackage{enumitem}

\usepackage{xcolor}

\usepackage{color}
\definecolor{redish}{rgb}{0.5, 0.08, 0.25}
\definecolor{greenish}{rgb}{0.12, 0.42, 0.18 }
\definecolor{blueish}{rgb}{0.06, 0.10, 0.6}
\definecolor{darkgray}{rgb}{0.4, 0.4, 0.4}         

\theoremstyle{plain}
\newtheorem{theorem}{Theorem}
\newtheorem{lemma}[theorem]{Lemma}

\theoremstyle{definition}

\usepackage{epsfig}
\usepackage{algpseudocode}
\usepackage{amscd}
\usepackage{algorithm}
\usepackage{tikz}
\usetikzlibrary{arrows,decorations.markings,shapes,snakes}
\tikzset{cross/.style={cross out, draw=black, minimum size=2*(#1-\pgflinewidth), inner sep=0pt, outer sep=0pt},
cross/.default={1pt}}

\usepackage{capt-of}
\usepackage{subcaption}

\newcommand{\beq}{\begin{equation}}
\newcommand{\eeq}{\end{equation}}
\newcommand{\beqa}{\begin{eqnarray}}
\newcommand{\eeqa}{\end{eqnarray}}
\newcommand{\bal}{\begin{align}}
\newcommand{\eal}{\end{align}}
\newcommand{\bsp}{\begin{equation}\begin{split}}
\newcommand{\esp}{\end{split}\end{equation}}
\newcommand{\bit}{\begin{itemize}}
\newcommand{\eit}{\end{itemize}}
\newcommand{\ben}{\begin{enumerate}}
\newcommand{\een}{\end{enumerate}}
\newcommand{\nn}{\nonumber}
\newcommand{\AR}{\mathbb{R}}
\newcommand{\tr}{\mathrm{tr}}
\newcommand{\rank}{\mathrm{rank}}

\begin{document}

%%%%%%%%%%%%%%%%%%%%%%%%%%%%%%%%%%%%
\title{Recommender systems inspired by the structure of quantum theory}

\author{Cyril J.~Stark\footnote{Center for Theoretical Physics, Massachusetts Institute of Technology, 77 Massachusetts Avenue, Cambridge MA 02139-4307, USA}}

\date{\today}

\maketitle

\begin{abstract}

Physicists use quantum models to describe the behavior of physical systems. Quantum models owe their success to their \emph{interpretability}, to their relation to \emph{probabilistic models} (quantization of classical models) and to their \emph{high predictive power}. Beyond physics, these properties are valuable in general data science. This motivates the use of quantum models to analyze general nonphysical datasets. Here we provide both empirical and theoretical insights into the application of quantum models in data science. In the \emph{theoretical part} of this paper, we firstly show that quantum models can be exponentially more efficient than probabilistic models because there exist datasets that admit low-dimensional quantum models and only exponentially high-dimensional probabilistic models. Secondly, we explain in what sense quantum models realize a useful relaxation of compressed probabilistic models. Thirdly, we show that sparse datasets admit low-dimensional quantum models and finally, we introduce a method to compute hierarchical orderings of properties of users (e.g., personality traits) and items (e.g., genres of movies). In the \emph{empirical part} of the paper, we evaluate quantum models in item recommendation and observe that the predictive power of quantum-inspired recommender systems can compete with state-of-the-art recommender systems like SVD++ and PureSVD. Furthermore, we make use of the interpretability of quantum models by computing hierarchical orderings of properties of users and items. This work establishes a connection between data science (item recommendation), information theory (communication complexity), mathematical programming (positive semidefinite factorizations) and physics (quantum models).

\end{abstract}

%%%%%
\section{Introduction}

Recommendation is a key discipline in machine learning that aims at predicting which items (e.g., movies, books but also events) are liked by which people~\cite{resnick1994grouplens,deshpande2004item,sarwar2001item,koren2009matrix}. Algorithms to compute these predictions are called \emph{recommender systems}. Ideally recommender systems address the following three objectives: (1) predictive power, (2) computational tractability, and (3) interpretability (e.g., to compute visual representations). Here we address these challenges by adopting the \emph{system-state-measurement paradigm~\cite{stark2015expressive}} in the form of a class of models (\emph{quantum models}) physicists use to describe quantum systems~\cite{nielsen2010quantum}.

Respecting the system-state-measurement paradigm amounts to distinguishing the `state of a system' from the `measurement device' we use to examine that system. For example, in recommendation, the \emph{system} is that abstract part of our mind that decides whether we like or dislike an item. Preferences vary from person to person. The correspondingly varying manifestation of a person's system system is described by the \emph{state} of the system. \emph{Measurements} are questions like ``Do you like the movie \emph{Despicable Me}?". Performing measurements on the taste of a person, we can get an increasingly refined understanding of a person's preferences, i.e., we get an increasingly refined understanding of a person's state. Hence, states and measurements are examples for user and item representations. In quantum models, both the states (user representations) and the measurements (item representations) are described in terms of normalized positive semidefinite (psd) matrices. We motivate the use of quantum models in terms of the following five points.

\begin{itemize}
\item		We show that quantum models realize convenient relaxations of both compressed (section~\ref{sect:compression.of.NNMs.2.NPSDs}) and uncompressed (section~\ref{Conservation.of.optimal.solutions.under.relaxation}) probabilistic models (see section~\ref{Sect:Normalized.nonnegative.models}). 
\item		Let $d$ be the dimension of the lowest-dimensional probabilistic model fitting a given dataset, and let $d_{\mathrm{Q}}$ be the dimension of the lowest-dimensional quantum model for the same dataset. Then, $d_{\mathrm{Q}} \leq d$ always. On the other hand, using a results on communication complexity, we show can show that there exist datasets where $d$ is exponential in $d_{\mathrm{Q}}$ (see section~\ref{sect:gap}). Hence, for some datasets, we cannot hope to be able to compute probabilistic models but we can hope to find low-dimensional quantum models.
\item		The close relationship between probabilistic and quantum models allows us to `quantize' methods developed for probabilistic models. We demonstrate this by `quantizing' a technique from~\cite{stark2015expressive} to compute hierarchical orderings of properties (e.g., tags) of items and users (see section~\ref{sect:extraction.of.hierarchy}). 
\item The success of quantum models in physics is largely due to their interpretability which is weaker than the interpretability of probabilistic models but greater than the interpretability of general matrix factorizations. This observation combined with $d_{\mathrm{Q}} \leq d$ (always) and $d_{\mathrm{Q}} \ll d$ on some datasets suggests that \emph{quantum models realize a practical compromise between high interpretability and low dimensionality of representations (e.g., user and item representations).}
\item 	We demonstrate the \emph{predictive power} of quantum models empirically (see section~\ref{sect:numerical.eval}). We observe that with respect to mean-average-error and recall, quantum models outperform major recommender systems like SVD++~\cite{koren2008factorization} and PureSVD~\cite{cremonesi2010performance} on MovieLens datasets. 
\end{itemize}

The remainder of this paper has four parts. In part~1 we provide preliminaries (secions~\ref{sec:notation} and \ref{Sect:Normalized.nonnegative.models}). In part~2 we define quantum models, provide a simple algorithm for their computation and summarize related work (sections~\ref{Sect:NPSD},~\ref{sect:computation.of.NPSDs} and~\ref{sec:related.work}). In part~3 we derive properties of quantum models (section~\ref{sect:compression.of.NNMs.2.NPSDs} to section~\ref{sect:extraction.of.hierarchy}). In part~4 we provide empirical results evaluating the performance of quantum models on movieLens datasets (section~\ref{sect:numerical.eval}). In part~5 we conclude the paper (section~\ref{sect:conclusion}).

%%%%%
\section{Notation}\label{sec:notation}

For any $n \in \mathbb{N}$ we denote by $[n]$ the set $\{1,...,n\}$. Throughout, we consider item recommendation. Here, $u \in [U]$ labels users, $i \in [I]$ labels items and $z \in [Z]$ denote possible ratings users provide for items (e.g., $z \in [5]$ in case of 5-star ratings). By $R \in [Z]^{U \times I}$ we denote the complete rating matrix, i.e., $R_{ui} \in [Z]$ is the rating that user $u$ provides for item $i$. In practice, we only know a subset of the entries of $R$. We use $\Gamma \subseteq [U] \times [I]$ to mark the known entries of $R$ at the time of data analysis. In the evaluation of recommender systems we use $\Gamma_{\text{train}} \subseteq [U] \times [I]$ to mark entries in the training set, and we use $\Gamma_{\text{test}} \subseteq [U] \times [I]$ to mark entries in the test set. A finite probability space is described in terms of a sample space $\Omega = \{ \omega_1, ..., \omega_D \}$, probability distributions $\vec{p}$ and random variables $\hat{E}$ with some alphabet $[Z]$. We use $\Delta = \{ \vec{p} \in \AR^D_+ | \| \vec{p} \|_1 = 1 \}$ to denote the unit simplex.  We denote by $\mathbb{P}_u[ \hat{E}_i = z ]$ the probability for user $u$ to rate item $i$ with $z \in [Z]$. For a matrix $A$, $\| A \|_{\mathrm{F}}$ denotes its Frobenius norm, $\| A\|_1$ denotes the nuclear norm, and $\| A \|$ denotes its operator norm. We denote by $\mathrm{spect}(A)$ the eigenvalues of $A$. We use $S^+(\mathbb{C}^D)$ to refer to the set of hermitian complex $D \times D$ matrices which are positive semidefinite (psd). $I$ denotes the identity matrix.

%%%%%
\section{Preliminaries}\label{Sect:Normalized.nonnegative.models}

By Kolmogorov, random (finite) experiments are described in terms of the following three parts. Firstly, a sample space $\Omega = \{ \omega_1, ..., \omega_D \}$. Secondly, a probability distribution $\vec{p} \in \Delta = \{ \vec{p} \in \AR^D_+ | \| \vec{p} \|_1 = 1 \}$. Thirdly, a random variable $\hat{E}$, i.e., a \emph{function} $\hat{E} : \Omega \rightarrow [Z]$. In the following, $\mathbb{P}[ \hat{E} = z ]$ denotes the probability for measuring the event $\{ \omega | \hat{E}(\omega) = z \} = \hat{E}^{-1}(z)$. A random variable $\hat{E}$ is fully specified by indicator vectors $\vec{E}_z \in \{ 0,1 \}^D$ whose supports equal $\{ \omega | \hat{E}(\omega) = z \} = \hat{E}^{-1}(z)$. Thus, these indicator vectors satisfy $\sum_z \vec{E}_z = (1,...,1)^T$. We observe that $\mathbb{P}[ \hat{E} = z ] = \vec{E}_z^T \vec{p}$.

In this section we show (at the example of item recommendation) how the system-state-measurement paradigm can be adopted in the form of normalized nonnegative models (NNM)~\cite{stark2015expressive}. In the application of NNMs, the system is described by some \emph{sample space} $\Omega = \{ \omega_1, ..., \omega_D \}$. Each user $u$ is represented in terms of a probability distribution $\vec{p}_u \in \Delta$ on $\Omega$, and answers to questions ``Do you like item $i$?" are regarded as samples of a random variable $\hat{E}_i: \Omega \rightarrow [Z]$ assigned to item $i$. In case of 5-star ratings, $Z=5$. In the remainder, $\mathbb{P}_u[ \hat{E}_i = z ]$ denoted the probability for user $u$ to rate $i$ with value $z \in [Z]$. Therefore, 
\beq\label{few45jkljfnk}
	\mathbb{P}_u[ \hat{E}_i = z ] = \mathbb{P}_u[ \hat{E}^{-1}_i(z) ] = \vec{E}_{iz}^T \vec{p}_u
\eeq
where $\vec{E}_{iz} \in \{0,1\}^D$ is defined through
\beq\label{fej435hjhj}
	\bigl( \vec{E}_{iz} \bigr)_j 
	= \left\{ \begin{array}{ll}  1,   & \text{ if $\omega_j \in \hat{E}^{-1}_i(z)$},      \\ 0,  &\text{ otherwise.}   \end{array} \right.
\eeq
By construction, $\sum_z \vec{E}_{iz} = (1,...,1)^T$. Let $\mathcal{M}_0$ be the set of allowed values of $( \vec{E}_{1}, ..., \vec{E}_{Z} )$, i.e., $\mathcal{M}_0 = \Bigl\{ ( \vec{E}_{1}, ..., \vec{E}_{Z} ) \in \{0,1\}^{D \times Z } \Bigl|  \sum_z \vec{E}_{iz} = (1,...,1)^T \Bigr\}$. We call the collection of vectors $\bigl( (\vec{p}_u)_u, (\vec{E}_{iz})_{iz} \bigr)$ a \emph{Kolmogorov factorization} if $\vec{p}_u \in \Delta$ for all users $u \in [U]$ and if $(\vec{E}_{iz})_{z \in [Z]} \in \mathcal{M}_0$ for all items $i \in [I]$ (see~\cite{stark2015expressive}). The convex relaxation of $\mathcal{M}_0$ is
\[
	\mathcal{M} := \Bigl\{ ( \vec{E}_{1}, ..., \vec{E}_{Z} ) \in \AR_+^{D \times Z } \Bigl|  \sum_z \vec{E}_{iz} = (1,...,1)^T \Bigr\}.
\] 
The tuple of vectors $\bigl( (\vec{p}_u)_u, (\vec{E}_{iz})_{iz} \bigr)$ is a \emph{normalized nonnegative model} if $\vec{p}_u \in \Delta$ for all users $u \in [U]$ and if $(\vec{E}_{iz})_{z \in [Z]} \in \mathcal{M}$ for all items $i \in [I]$ (see~\cite{stark2015expressive} for simple examples). 

NNMs allow for interpreting user preferences as mixture of interpretable user stereotypes~\cite{stark2015expressive}, they allow for the computation of hierarchical orderings of tags of users and items~\cite{stark2015expressive}, they allow for fast inference of users' states in terms of interpretable questions~\cite{stark2015amortized}, and NNMs can be used to compute operational user-user and item-item distance measures~\cite{stark2015top}. Consequently, the relaxation of Kolmogorov factorizations $\mapsto$ NNMs preserves a lot of the interpretability of Kolmogorov factorizations.

\subsection{Categorical variables}\label{sect:noncat}

In the definition of NNMs, random variables are categorical. Hence, NNMs can in principle be used to fit a wide range of data. For instance, in recommendation, side information about the occupation of users can be modeled in terms of a categorical random variable.

However, we do not make use of all available information when treating ordered data in a categorical manner. For instance, star ratings (provided by users for items) are ordered because a 4-star rating is better than a 3-star rating. Hence, if we want to use data economically, it can be beneficial to interpret ratings $R_{ui} \in [Z]$ of an item $i$ provided by a user $u$ as approximation of $\mathbb{P}[u \text{ likes } i]$. More precisely,
\beq\label{Eq:alternative.interpretation.of.data}
	R_{ui}/Z \approx \mathbb{P}[u \text{ likes } i].
\eeq
Then, the random variables $\hat{E}_i$ are binary. Their outcomes are interpreted as `\emph{I like item $i$}' and `\emph{I dislike item $i$}', respectively.

%%%%%
\section{Quantum models}\label{Sect:NPSD}

By~\eqref{few45jkljfnk}, NNMs are specified in terms of nonnegative vectors $\in \AR^D_+$. Equivalently, we could describe NNMs through diagonal psd matrices:
\beq\begin{split}\label{fekj45k653}
	\vec{p}_u 		&\mapsto \sum_{j} (\vec{p}_u)_j  \vec{e}_j \vec{e}_j^T =: \rho^{(0)}_u, \\ 
	\vec{E}_{iz} 	&\mapsto \sum_{j} (\vec{E}_{iz})_j \vec{e}_j \vec{e}_j^T =: E^{(0)}_{iz}
\end{split}\eeq
where $(\vec{e}_j)_n = \delta_{jn}$ denotes the canonical basis. The two descriptions $\vec{p}_u, \vec{E}_{iz}$ and $\rho^{(0)}_u, E^{(0)}_{iz}$ are entirely equivalent if 
\beq\label{fej4h35hfen}
	\mathbb{P}_u[ \hat{E}_i = z ] = \sum_{ij} (\rho^{(0)}_u)_{ij} (E^{(0)}_{iz})_{ij} = \tr(\rho^{(0)}_u E^{(0)}_{iz}). 
\eeq
But the description $\rho^{(0)}_u, E^{(0)}_{iz}$ motivates a natural \emph{relaxation} of NNMs by relaxing the constraint that the positive semidefinite (psd) matrices $\rho^{(0)}_u, E^{(0)}_{iz}$ are diagonal. This relaxation is sometimes called `quantization' in the physics literature. After quantization, the state of a user $u$ is represented by $\rho_u \in \Delta'$,
\beq\label{fewk43afe5hjA}
	\Delta' = \bigl\{ \rho \in S^+(\mathbb{C}^D) \bigl| \tr(\rho) = 1 \bigr\},
\eeq
and an item $i$ is represented by $(E_{iz})_{z \in [Z]} \in \mathcal{M}'$,
\beq\label{fewk43afe5hjB}
	\mathcal{M}' =		\Bigl\{ (E_z)_{z \in [Z]} \in S^+(\mathbb{C}^D)^Z \Bigl| \sum_z E_z = \mathbb{I} \Bigr\}.
\eeq
In quantum information~\cite{nielsen2010quantum}, $\Delta'$ is the space of so called \emph{quantum states} (aka \emph{density matrices}) and $\mathcal{M}'$ is the space of \emph{quantum measurements} (aka \emph{positive operator valued measure}). For this reason, we call the tuple of user and item matrices $\bigl( (\rho_u)_u, (E_{iz})_{iz} \bigr)$ a \emph{quantum model} if $\rho_u \in \Delta'$ for all users $u \in [U]$ and if $(E_{iz})_{z \in [Z]} \in \mathcal{M}'$ for all items $i \in [I]$. Inspired by~\eqref{fej4h35hfen}, the rule
\beq\label{fej4h35hfen}
	\mathbb{P}_u[ \hat{E}_i = z ] = \sum_{ij} \overline{(\rho_u)}_{ij} (E_{iz})_{ij} =  \tr(\rho_u E_{iz})
\eeq
relates user and item representations to observable probabilities. Here, $\bar{z}$ denotes complex conjugation of $z$. Up to normalization constraints, models in terms of $\rho_u$ and $E_{iz}$ are intimately related to psd factorizations~\cite{fiorini2012linear,gouveia2013lifts} studied in mathematical programming; see~\cite{fawzi2014positive} for a recent review. 

\emph{Example.} Let $D = 2$ and $Z = 2$. Then, for 
\[
	\rho_u = \frac{1}{5} \left( \begin{array}{cc}
				1  & 2    \\
				 2 & 4   
				\end{array}\right), \;
	E_{i1} = \frac{1}{10} \left( \begin{array}{cc}
				1 & 3    \\
				 3 & 9   
				\end{array}\right), \;
	E_{i2} = \frac{1}{10} \left( \begin{array}{cc}
				9 & -3    \\
				 -3 & 1   
				\end{array}\right),			
\]
we get
\[
	\mathbb{P}_u[ \hat{E}_i = 1 ] = \tr(\rho_u E_{i1}) = \frac{49}{50}, \; \mathbb{P}_u[ \hat{E}_i = 2 ] = \tr(\rho_u E_{i2}) = \frac{1}{50},. 
\]

%%%%%
\section{Heuristic computation of quantum models}\label{sect:computation.of.NPSDs}

We denote by $R_{ui} \in [Z]$ the rating user $u$ provides for item $i$. Let $\Gamma \subseteq [U] \times [I]$ be such that $(u,i) \in \Gamma$ if and only if $R_{ui}$ is known. We adopt~\eqref{Eq:alternative.interpretation.of.data} and interpret $z=1$ as `like' and $z=2$ as `dislike'. Then, the computation of NNMs amounts to approximately solving some variant of the optimization problem
\beq\label{fejw4h5jh}
\begin{split}
	\mathrm{minimize} 	& \ \ \sum_{(u,i) \in \Gamma} \bigl( \vec{E}_{i1}^T \vec{p}_u - R_{ui}/Z \bigr)^2 \\ 
	\text{s.t.}  				& \ \ \ \vec{p}_u \in \Delta, (\vec{E}_{i1},\vec{E}_{i2}) \in \mathcal{M}
\end{split}
\eeq
Similarly, the computation of quantum models amounts to solving a variant of
\beq\label{fewk43gssae5hdwdj}
\begin{split}
	\mathrm{minimize} 	& \ \ \sum_{(u,i) \in \Gamma} \bigl( \tr(E_{i1} \rho_u) - R_{ui}/Z \bigr)^2 \\ 
	\text{s.t.}  				& \ \ \ \rho_u \in \Delta', (E_{i1},E_{i2}) \in \mathcal{M}'
\end{split}
\eeq
A simple approach to approximately solve~\eqref{fewk43gssae5hdwdj} is through alternating constrained optimization where we distinguish between the update of the user representation and the update of the item representation. When updating the users we keep the items fixed and when updating the items we keep the users fixed. 

Choosing a convenient initialization of the user matrices/vectors improves the performance of alternating optimization. Next, we describe one particular possibility to find convenient initializations. Our initialization-strategy is based on the observation that real-life data is oftentimes subject to a strong selection-bias towards high ratings; see for example figure~\ref{fig:fig_about_selection_bias} about the MovieLens 1M dataset.\footnote{http://files.grouplens.org/datasets/movielens/ml-1m-README.txt} Selection biases of that kind can be detrimental for item recommendation~\cite{steck2011item}. To illustrate this, we imagine that all the provided ratings are 5-star ratings. Then, the 1(!)-dimensional factorization with $p_u = 1$ and $E_{i1} = ... = E_{i4} = 0$, $E_{i5} = 1$ fits the data perfectly but is expected to have extremely poor predictive power because it predicts that all users like all items.

\begin{figure}[tbp]
\centering
\includegraphics[width=0.4\columnwidth]{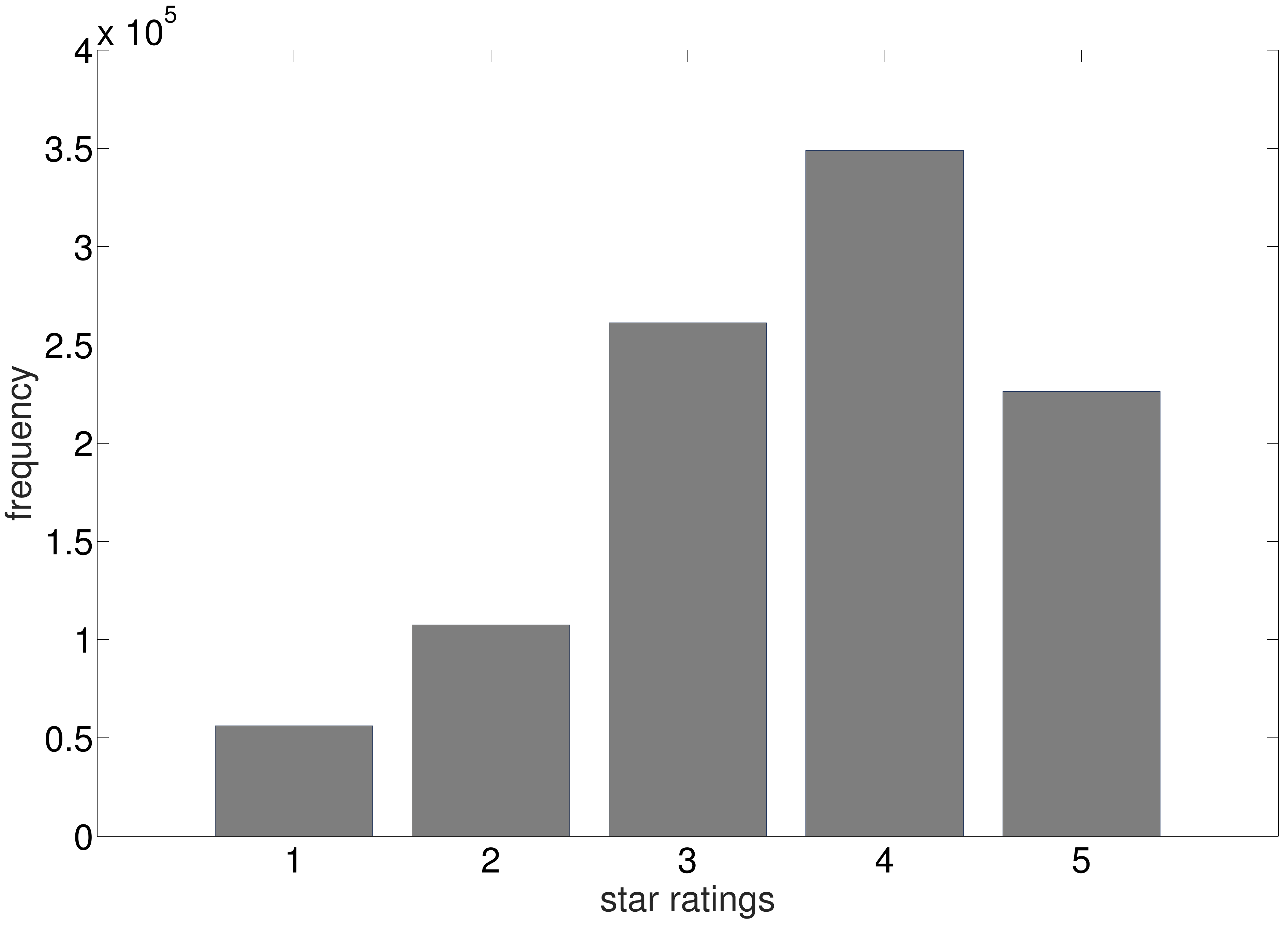}
\caption{Frequency of ratings as a function of $z \in [5]$ at the example of the MovieLens 1M dataset.}
\label{fig:fig_about_selection_bias}
\end{figure}

The selection bias is a consequence of the simple fact that consumers usually only consider watching a movie which they expect to like. Hence, conditioned on a movie being watched, it is more likely to be rated highly. On the other hand, if users have not yet watched a movie, then this suggests that they would rate the movie badly if they watched it. Therefore, we should expect a movie $i$ to be rated badly by user $u$ if $u$ has not yet rated $i$. This side-information can be used to improve the performance of recommender systems~\cite{cremonesi2010performance}. 

How we make use of knowing about the selection-bias depends on the figure of merit. When aiming for large \emph{recall},\footnote{The appendix of~\cite{stark2015expressive} contains a reminder of the definitions of recall, MAE and RMSE.} we suggest to assume that all the unknown entries of the user-item matrix $R$ are zero when approximating~\eqref{fewk43gssae5hdwdj}---the precise star value does not matter too much and those zero-fillings capture a general trend. On the other hand, when aiming for small \emph{mean-average-error (MAE)}, we should not proceed in exactly the same manner because setting unknown entries equal to zero can distort the precise value of the predicted star ratings too much. This is why we suggest to only use the zero-fillings during initial iterations of alternating optimization (to end up with an initialization of the users and items that is better than a random initialization). We observe empirically that said strategies for initialization indeed improve predictive power. Zero-fillings of this kind are common in the design of recommender systems~\cite{cremonesi2010performance}.

Adopting~\eqref{Eq:alternative.interpretation.of.data}, we thus arrive at  Algorithm~\ref{Alg:Constrained.least.squares.for.NPSD} for solving~\eqref{fewk43gssae5hdwdj}. Some practically-oriented readers might refrain from implementing Algorithm~\ref{Alg:Constrained.least.squares.for.NPSD} because it requires the solution of SDPs. However, for the purpose of recommendation, we do not require a full-fledged SDP solver like SDPT3~\cite{toh1999sdpt3}---a simple barrier function to enforce the SDP constraint approximately is sufficient because we are not primarily interested in high precision. 

\begin{algorithm}
\caption{Alternating optimization for quantum models}\label{Alg:Constrained.least.squares.for.NPSD}
\begin{algorithmic}[1]
\State 	Fix $D$ (e.g., by cross validation).
\State	For all $u$, initialize $\rho_u = \vec{v}_u^T \vec{v}_u$ where $\vec{v}_u$ is sampled uniformly from the complex unit sphere.
\State	For all items $i$, solve the semidefinite program $\min_{(E_{i1},E_{i2}) \in \mathcal{M}'} \sum_{u: (u,i)\in\Gamma} \bigl( \tr({E}_{i1} \rho_u) - R_{ui}/Z \bigr)^2$.\label{fewkjh45h}
\State	For all users $u$, solve the semidefinite program $\min_{\rho_u \in \Delta'}  \sum_{i: (u,i)\in\Gamma} \bigl( \tr({E}_{i1} \rho_u) - R_{ui}/Z \bigr)^2$.\label{fjhHJb3}
\State 	Repeat steps 3 and 4 until a stopping criteria is satisfied; e.g., until a maximum number of iterations is reached. For the first 2 iterations, \emph{when optimizing for MAE}, we pretend we knew all of $R$ by setting unknown entries equal to zero. \emph{When optimizing for recall}, we use those zero-fillings during all iterations.
\end{algorithmic}
\end{algorithm}

%%%%%
\section{Related work}\label{sec:related.work}

{\bf Data mining.} In the study of recommendation we are interested in events like \emph{`$u$ buys item $i$'}, \emph{`$u$ attends event $i$'}, etc. We abbreviate these events by $(u,i)$. In \emph{aspect models}, \emph{pLSA}~\cite{hofmann1999latent,hofmann1999probabilistic} (and similarly in \emph{latent Dirichlet allocation}~\cite{blei2003latent}) we regard the value $(u,i)$ as random variable with distribution
\beq\label{fwfefeg}
	\mathbb{P}[u,i] = \sum_{k} \; \mathbb{P}[k] \; \mathbb{P}[u|k] \; \mathbb{P}[i|k].
\eeq
Thus, when adopting~\eqref{fwfefeg}, we model $\mathbb{P}[u,i]$ as (possibly scaled) inner product between vectors $(\mathbb{P}[u|k])_k$ and $(\mathbb{P}[i|k])_k$. This is reminiscent of the inner product~\eqref{fej4h35hfen} where one half of the inner product is the non-commutatitive analog of a probability distribution associated to $u$. However, \eqref{fwfefeg} and~\eqref{fej4h35hfen} disagree on the second half of said inner product. In~\eqref{fwfefeg}, the second factor is another probability distribution. In~\eqref{fej4h35hfen} on the other hand, the second factor $E_{iz}$ is constrained through the existence of $E_{i1}, ..., E_{iz-1}, E_{iz+1}, .., E_{iZ}$ such that $(E_{i1}, ..., E_{iZ}) \in \mathcal{M}'$. 

In a similar manner, quantum models are related to more general NMF-based models~\cite{ma2011recommender,hyvonen2008interpretable} because in both approaches we model data in terms of cone-constraint vectors: nonnegative vectors in case of NMF and psd matrices in case of quantum models. 

The description of data in terms of $\rho_u$, $E_{iz}$ allows for the description of categorical random variables (see section~\ref{sect:noncat}) and we can straightforwardly extract hierarchical structures from quantum models; see section~\ref{sect:extraction.of.hierarchy}. Aspect models can also be used for the description of categorical variables~\cite{hofmann1999latent} and for the computation of hierarchical structures~\cite{blei2010nested,paisley2015nested}. Quantum models realize an alternative approach for meeting these objectives. Furthermore, categorical random variables can also be modeled by a family of graphical models called \emph{multinomial mixture models}~\cite{marlin2003modeling}. These model class puts forward a single distribution $\theta$ (only one distribution for all users) over so called user attitudes. Thus, in contrast to NNMs and quantum models, the prediction for user $u$'s rating of item $i$ is not a function of an individual user vector and an item vector. We regard the separate description of users and items to be important because it is one pillar of interpretability of models like NNMs and quantum models.

\emph{Probabilistic matrix factorization} (PMF,~\cite{mnih2007probabilistic}) is a very practical class of models related to quantum models. In PMF we regard entries $R_{ui}$ of the rating matrix as independent Gaussian random variables with mean $\vec{U}_u^T \vec{V}_i$ and variance $\sigma$. Therefore, as in case of aspect models and NMF, the description of ratings is the result of inner products between some vectors. In contrast to PMF, quantum models do not need to assume a Gaussian distribution for ratings. Once the dimension $D$ has been fixed by cross validation, we do not need to assume a particular family of distributions for $R_{ui}$.

\emph{Bayesian nonnegative matrix factorization (BNMF)}~\cite{paisley2015bayesian} is closely related to PMF because ratings are sampled from a particular parametric family of distributions and because we impose a low-rank (also nonnegative in BNMF) model onto the parameters. BNMF assumes Poisson-distributed ratings. Hence, using BNMF we can model discrete but not categorical random variables. 

In Algorithm~\ref{Alg:Constrained.least.squares.for.NPSD} we make use of the selection bias in recommendation datasets. Selection biases have been studied in more detail in the paper~\cite{steck2011item}. This paper introduces the recommendation system AllRank which successfully accounts for the selection bias. We could have also included $r_m$ and $w_m$ into the training of quantum models but we felt that this would distract from the simplicity of inference of quantum models through alternating optimization. AllRank allows for great predictions but not for strong interpretability.

{\bf Quantum theory.} In physics we distinguish between two different ways to describe quantum experiments. The \emph{first possibility} is simply to describe what we do and what we observe. For instance, we could describe an experiment in terms of two manuals and a dataset. The first manual provides an experimentalist with all the necessary instructions to prepare some quantum states and to apply some processes. The second manual specifies how to build and apply some measurement devices. The dataset is a record of measurement results. The \emph{second approach} to describe quantum experiments is in terms of density matrices, measurement matrices and processes~\cite{nielsen2010quantum}. This mathematical description allows for predictions of future measurement outcomes, and it makes the experiment highly interpretable. 

A significant part of physics is about inference, namely, the translation of the first empirical description (i.e., the two manuals plus dataset) into the second mathematical description (i.e., states, measurements and processes). Here again we distinguish between two different approaches. The more traditional approach uses physical heuristics to translate the manuals and the dataset into a theoretical model of the experiment (e.g., through quantization of classical models). However, in the past two decades we observed limits of physical heuristics as the complexity of modern quantum experiments (e.g., superconducting qubits) appears to be too high to be accurately described in terms of heuristics~\cite{merkel2012self,kimmel2014robust}. The resulting absence of accurate theoretical models for existing quantum devices presents a severe bottleneck for the design of new quantum devices and quantum computers. This deficiency led to the development of methods for self-consistent tomography~\cite{stark2014self,rosset2012imperfect,navascues2015non,navascues2008convergent,stark2014compressibility,gallego2010device,harrigan2007representing,wehner2008lower,stark2012global,jackson2015detecting,merkel2012self,blume2013robust,kimmel2014robust,kimmel2015quantum,johnson2015demonstration,kimmel2015robust,schwemmer2015systematic,ferrie2014quantum,Dugas2014characterizing,monras2014quantum,sikora2015minimum,jackson2015detecting,greenbaum2015introduction} where we try to infer quantum models in a way that avoids physical heuristics. In self-consistent tomography, an experiment is regarded as a black box that accepts settings (specifying the state to prepare, the process to apply, and the measurement to perform) and outputs measurement results. This black box interpretation opens up the possibility to use methods from self-consistent tomography to fit quantum models to \emph{arbitrary non-physical datasets}. The quantum models we obtain can then be used to interpret data through a generalization probability theory: quantum theory. Here, in this paper, we introduce this quantum perspective on general datasets at the example of item recommendation.

The paper~\cite{stark2015learning} shows that exact inference of lowest-dimensional quantum models is \emph{NP}-hard.

Here, in this paper, we are using structural properties underlying quantum models. Alternatively, we can try to build quantum devices to run quantum algorithms for machine learning~\cite{anguita2003quantum,servedio2004equivalences,aimeur2006machine,rendle2009bpr,neven2009training,aimeur2013quantum,pudenz2013quantum,lloyd2013quantum,rebentrost2014quantum,wiebe2014quantum,lloyd2014quantum}. We observe a quickly growing interest in this new and promising area at the intersection of physics and machine learning. 

{\bf Quantum foundations.} We argue that quantum models realize a useful tradeoff between the low complexity of general matrix factorizations and the high interpretability of NNMs; see figure~\ref{Fig:dim.interpretability.tradeoff}. This discussion fits naturally into the discipline called quantum foundations where researchers try to determine the operational consequences of different choices of state and measurement spaces~\cite{hardy2001quantum,Aaronson2007,muller2012structure,gross2010all,barnum2012teleportation,colbeck2011no,pfister2013information,ried2015quantum}. The significance of the comparison of different regularization schemes in data science confirms the practical value of results produced by researchers working on quantum foundations. Here we give an argument for quantum models in terms of the operational meaning of interpretability. Unfortunately, interpretability appears to be hard to capture axiomatically. Consequently, it is not clear how we can build axioms for quantum theory on the basis of interpretability.

{\bf Cognition and economics.} Item recommendation is related to the study of cognition and economics because item recommendation aims at capturing peoples preferences. In these disciplines, quantum models have become increasingly popular in the past years; see for instance~\cite{busemeyer2012quantum,haven2013quantum,yukalov2009processing,sornette2014physics,aerts2013concepts} and references therein.

%%%%%
\section{Quantum models as compression of NNMs}\label{sect:compression.of.NNMs.2.NPSDs}

Considering $Z$ and the sparsity of item vectors $\vec{E}_{iz}$ as a function of $D$, we call a NNM \emph{pseudo sparse} if
\begin{itemize}
\item	$Z$ small, i.e., $Z \sim \mathcal{O}(1)$ in $D$
\item For all items $i$ there exists $z' \in [Z]$ such that $$\underbrace{\vec{E}_{y,1}, ..., \vec{E}_{y,z'-1}}_{\text{sparse}}, \vec{E}_{yz'}, \underbrace{\vec{E}_{y,z'+1}, ..., \vec{E}_{y,Z}}_{\text{sparse}}$$ where `sparse' means `sparsity $\sim \mathcal{O}(1)$ in $D$'.
\end{itemize}
Similarly, we call a quantum model $\bigl( (\rho_{u})_{u}, (E_{iz})_{iz} \bigr)$ \emph{pseudo-low rank}~\cite{stark2014compressibility} if
\begin{itemize}
\item	$Z$ small, i.e., $Z \sim \mathcal{O}(1)$ in $D$
\item For all items $i$ there exists $z' \in [Z]$ such that $$\underbrace{E_{y,1}, ..., E_{y,z'-1}}_{\text{low-rank}}, E_{yz'}, \underbrace{E_{y,z'+1}, ..., E_{y,Z}}_{\text{low-rank}}$$ where `low rank' means `rank $\sim \mathcal{O}(1)$ in $D$'.
\end{itemize}
The following Theorem~\ref{cor:fj345ewklnnr} shows that pseudo-sparse NMMs can be compressed into low-dimensional quantum models.

\begin{theorem}\label{cor:fj345ewklnnr}
	Let $J = U + IZ$, let $\varepsilon$ satisfy $\frac{4}{\sqrt{JD}} \leq \varepsilon \leq \frac{1}{2}$ and let $\bigl((\vec{p}_{u})_{u}, (\vec{E}_{iz})_{iz} \bigr)$ be a $D$-dimensional NNM. Assume
	\begin{itemize}
				\item[]	$\bigl((\vec{p}_{u})_{u}, (\vec{E}_{iz})_{iz} \bigr)$ is \emph{pseudo sparse}, and			
				\item[]	$d = \mathcal{O}\Bigl( \frac{1}{\varepsilon^2} \; \ln\bigl( 4JD \bigr) \Bigr)$,
	\end{itemize}
	Then, there exists a $d$-dimensional quantum model $\bigl( (\rho'_{u})_{u}$ $(E'_{iz})_{iz} \bigr)$ satisfying $\bigl| \vec{p}^T_{u} \vec{E}_{iz} - \tr(\rho'_{u}E'_{iz}) \bigr| \leq \mathcal{O}(\varepsilon)$ for all $u,i,z$.
\end{theorem}

We prove Theorem~\ref{cor:fj345ewklnnr} in the appendix. By Theorem~\ref{cor:fj345ewklnnr}, $\Delta'$ and $\mathcal{M}'$ (see section~\ref{Sect:NPSD}) realize relaxations of the compression of $\Delta$ and $\mathcal{M}$ (see section~\ref{Sect:Normalized.nonnegative.models}). Semidefinite programming allows for the efficient optimization over $\Delta$ and $\mathcal{M}$. Of course, we could also characterize compressed NNMs in terms of the polyhedral cone $C$ which is the image of $\AR^D_+$ under a Johnson-Lindenstrauss-type compression mapping~\cite{johnson1984extensions}. However, that cone has \emph{exponentially many faces} when $D$ is exponential in $d$. The time complexity of linear programming increases polynomially in the number of half-space constraints. It follows that optimizing over said polyhedral cone $C$ is very expensive. Hence, we cannot reconstruct a compressed NNM by running methods like alternating optimization on $C$. However, we can optimize \emph{efficiently over the psd cone from Corollary~\ref{cor:fj345ewklnnr} which realizes a relaxation of $C$}. Hence, running methods like alternating optimization on the psd cone, we can compute relaxed compressed NNMs.

%%%%%
\section{Fitting sparse data}\label{sect:overfitting}

Quantum models would not be of any use  if they could only be fitted to a small class of empirical data. Here we show that this is not the case by proving that quantum models can be fitted to sparse datasets. This is significant because most real-world datasets are sparse.\footnote{http://www.librec.net/datasets.html} 

Recall that $\Gamma \subseteq [I] \times [Z]$ marks the positions of known entries of the rating matrix $R$. By the following Theorem~\ref{thm:fkerjtk4jk}, we can find \emph{low}-dimensional quantum models fitting $R_{\Gamma}$ with the property that for all $(u,i) \in \Gamma$,
\beq\begin{split}
	\tr(\rho_{u}E_{iz})
	\left\{\begin{array}{ll}
  		\geq 1 - \varepsilon, 	& \text{ if $R_{ui} = z$},      \\
  		\leq \varepsilon,		& \text{ otherwise,} 
	\end{array}\right.
\end{split}\eeq
if $\Gamma$ is sparse.

\begin{theorem}\label{thm:fkerjtk4jk}
	 Denote by $U$ and $I$ the number of users and items, respectively. Let $J = U + IZ$ and let $\varepsilon$ satisfy $\frac{4}{\sqrt{JD}} \leq \varepsilon \leq \frac{1}{2}$. Assume that the number of times an item has been rated is constant in $U$. Then,  for
	\[
		d = \mathcal{O}\Bigl( \frac{1}{\varepsilon^2} \; \ln\bigl( 4JU \bigr) \Bigr),
	\]
	there exists a $d$-dimensional quantum model $\bigl( (\rho_{u})_{u}$ $(E_{iz})_{iz} \bigr)$ satisfying $\bigl| \delta_{z,R_{ui}} - \tr(\rho_{u}E_{iz}) \bigr| \leq \mathcal{O}(\varepsilon)$ for all $(u,i) \in \Gamma$ and $z \in [Z]$.
\end{theorem}

By Theorem~\ref{thm:fkerjtk4jk}, there exist low-dimensional quantum model fitting sparse datasets. On the other hand, Theorem~\ref{thm:fkerjtk4jk} states that it is easy to overfit sparse datasets with quantum models. These low-dimensional models have \emph{no predictive power} despite being low-dimensional. This is because we end up with said low-dimensional quantum model by compressing a NNM that had no predictive power in the first place.\footnote{It predicts that with probability 1, all the missing entries of $R$ are equal to 1 (which is an arbitrary choice).} Theorem~\ref{thm:fkerjtk4jk} is proven in the appendix.\footnote{To our knowledge it is an open problem to prove whether or not a result analogous to Theorem~\ref{thm:fkerjtk4jk} holds for NNMs.}

%%%%%
\section{Exponential dimension gap between quantum models and NNM}\label{sect:gap}

For some $\varepsilon \in [0,1/2)$, set
\beqa
	d_{\mathrm{NNM}} := 
	&\min		& d \nn \\
	&\mathrm{s.t.}	& \text{$\exists$ $d$-dimensional NNM ($\vec{p}_u,\vec{E}_{iz})$ s.t.} \nn \\
	&			& \text{$\bigl| \vec{p}_u^T\vec{E}_{iz} - \delta_{z,R_{ui}} \bigr| \leq 1-\varepsilon$ $\forall$ $ui \in \Gamma$} \label{fekjlk5j}
\eeqa
and
\beqa 
	d_{\mathrm{Q}} := 
	&\min		& d \nn \\
	&\mathrm{s.t.}	& \text{$\exists$ $d$-dimensional q-model ($\rho_u,E_{iz})$ s.t.} \nn \\
	&			& \text{$\bigl| \tr(\rho_u E_{iz}) - \delta_{z,R_{ui}} \bigr| \leq 1-\varepsilon$ $\forall$ $ui \in \Gamma$.} \label{fekjlk5jfe}
\eeqa
The following Theorem~\ref{thm:gap.thm} is proven in the appendix using seminal results about the one-way communication complexity of Boolean functions~\cite{Gavinsky,montanaro2011new}.

\begin{theorem}\label{thm:gap.thm}
	For all data $(R_{ui})_{ui \in \Gamma}$, $d_{\mathrm{Q}} \leq d_{\mathrm{NNM}}$. Moreover, there exists $(R_{ui})_{ui \in \Gamma}$ such that
	\[
		\log_2(d_{\mathrm{Q}}) =  \mathcal{O}(\log(n)), \ \ \log_2(d_{\mathrm{NNM}}) = \Theta(\sqrt{n}).
	\]
\end{theorem}

By Theorem~\ref{thm:gap.thm} there exist partially known rating matrices $R_{\Gamma}$ for which $d_{\mathrm{NNM}}$ is exponential in $d_{\mathrm{Q}}$. Hence, there exist data tables that cannot be fitted by NNMs (for computational reasons) but can in principle be fitted by quantum models. In the appendix we prove Theorem~\ref{thm:gap.thm} through a relation between the dimension of NNMs (quantum models) and bounded error (quantum) one-way communication complexity of Boolean functions.

%%%%%
\section{Computation of NNMs through quantum models}\label{Conservation.of.optimal.solutions.under.relaxation}

Imagine we want to find a NNM satisfying $\mathbb{P}_u[\hat{E}_i = z] \approx \vec{p}_u^T \vec{E}_{iz}$. For this purpose we could run the natural analog of Algorithm~\ref{Alg:Constrained.least.squares.for.NPSD} where the SDP-constraint is replaced by a $\AR^D_+$-constraint. However, here we suggest to consider running Algorithm~\ref{Alg:Constrained.least.squares.for.NPSD} for the computation of NNMs because when passing from NNMs to quantum models we enlarge the parameter space that is available to heuristic algorithms like alternating optimization. In general, this is likely to be beneficial whenever \emph{minimizers of~\eqref{fejw4h5jh} can be extracted from minimizers of~\eqref{fewk43gssae5hdwdj}}. Here, in terms of the analysis of the following toy-example (Lemma~\ref{lem:fwerg435e}), we show that this can be the case. Enlarging feasible sets in that manner is bread-and-butter in non-convex optimization~\cite{gouveia2013lifts} where we try to `lift' difficult optimization problems to end up with easy optimization problems (ideally convex). 

\begin{lemma}\label{lem:fwerg435e}
	For each dimension $D$ there exist NNMs $\bigl( (\vec{p}_u)_u, (\vec{E}_{iz})_{iz} \bigr)$ with the following property. For every $D$-dimensional quantum model $\bigl( (\rho_u)_u, (E_{iz})_iz \bigr)$ satisfying $\vec{p}_u^T \vec{E}_{iz} = \tr(\rho_u E_{iz})$ (for all $u,i,z$) there exists a unitary matrix $U$ such that
	\[
		U\vec{p}_u \vec{p}_u^T U^* = \rho_u, \ \ \  U\vec{E}_{iz} \vec{E}_{iz}^T U^*
	\] 
	for all $u,i,z$. The unitary $U$ can be computed efficiently.
\end{lemma}

Lemma~\ref{lem:fwerg435e} shows that there exist situations where an underlying NNM can be computed by solving the relaxation~\eqref{fewk43gssae5hdwdj} of~\eqref{fejw4h5jh}. We think that the class of toy-examples provided in the proof of Lemma~\ref{lem:fwerg435e} successfully captures the nature of situations where the available data about the users' ratings is sufficient to push the user and item representations towards the boundary of $\AR^D_+$ and $S^+(\mathbb{C}^D)$ (so that uniqueness of any NNM and quantum model is enforced). We provide the proof of Lemma~\ref{lem:fwerg435e} in the appendix.

%%%%%
\section{Extraction of hierarchical orderings}\label{sect:extraction.of.hierarchy}

Here, we describe how quantum models can be used to order properties of items (or users) in a hierarchical manner; see for example figure~\ref{ML1M_hierarchy_graph}. 

%%%%%
\subsection{Description of tags}

In practice, we often have not only access to entries of the user-item matrix $R$ but we also have access to side information specifying properties of users and items. For instance, in movie recommendation, the MovieLens 1M dataset specifies the genre of each of the movies in the dataset. For example, \emph{Jurassic Park} $\in \{\text{Action, Adventure, Sci-Fi}\}$ and \emph{Anchorman} $\in \{\text{Comedy}\}$. Let $\{ i_1^t, ..., i^t_{m_t} \}$ be the set of items that have been tagged with a particular tag $\tau_t$ from the set $\{ \tau_t \}_{t \in [T]}$ of all tags. Is there a way to represent tags (e.g., genre tags) mathematically? 

To answer this question, we suggest to consider the following game which was introduced in~\cite{stark2015expressive}. It involves a referee and a player named Alice. The game proceeds as follows.
\begin{enumerate}
\item The referee chooses a tag $\tau_t$ and a user $u$.
\item Alice is given access to $\rho_u$, to $\tau_t$, to $\{ i_1^t, ..., i^t_{m_t} \}$ and to all item matrices $\vec{E}_{iz}$ ($Z = 2$; $z=1$ means `like', $z=2$ means `dislike').
\item The referee chooses uniformly at random an item $i^*$ from the set $\{ i_1^t, ..., i^t_{m_t} \}$ of all items that were tagged with $\tau_t$.
\item Querying $R$, the referee checks whether user $u$ likes or dislikes $i^*$. We denote $u$'s opinion by $z^* \in \{ \text{like,dislike} \}$.
\item Alice guesses $z^*$. She wins the game if she guesses correctly and loses otherwise.
\end{enumerate}

What is Alice's winning probability? We denote by $\mathbb{P}[ z^* = 1 | i ]$ the probability for the event $\{ z^* = 1 \}$ conditioned on the event \emph{Referee draws $i$}. It follows that $\mathbb{P}[ z^* = 1 | i ] = \tr(\rho_u E_{i1})$. The probability for the event \emph{Referee draws $i$} is $1/m_t$ because in total there are $m_t$ items tagged with $\tau_t$. Thus,
\[
	\mathbb{P}[ z^* = 1] = \sum_{i \in \{ i_1^t, ..., i^t_{m_t} \}} \mathbb{P}[ z^* = 1 | i ] \mathbb{P}[i] =  \tr(\rho_u E_{t})
\]
with
\beq\label{fwef454}
	E_t := \frac{1}{m_t} \sum_{i \in \{ i_1^t, ..., i^t_{m_t} \}} E_{i1}.
\eeq
$E_t$ serves as useful characterization of the item property $\tau_t$ because \emph{$E_t$ determines for each user $u$ the probability that $u$ likes a random item with property $\tau_t$}. Similarly, in NNMs, we can describe tags through tag-vectors $\vec{E}_t := \frac{1}{m_t} \sum_{i \in \{ i_1^t, ..., i^t_{m_t} \}} \vec{E}_{i1}$.

%%%%%
\subsection{Hierarchical ordering}\label{sect:hierarchical.ordering}

Let $\vec{p}_u, \vec{E}_{iz}$ be a NNM for some users and items, and let $\{ \vec{E}_t \}_{t \in [T]}$ denote tag-vectors. If the item vectors $\vec{E}_{iz}$ were indicator vectors as in~\eqref{fej435hjhj}, then it would be natural to define an ordering `$\subseteq$' on the set of items by
\beq\label{fwefwt3464gregr}
	\tau_t \subseteq \tau_{t'} \; :\Leftrightarrow  \; \mathrm{support}(\vec{E}_t) \subseteq \mathrm{support}(\vec{E}_{t'}).
\eeq
That is because $\text{$\vec{p}_u^T\vec{E}_t \approx 1$ implies $\vec{p}_u^T\vec{E}_{t'} \approx 1$}$ whenever $\mathrm{support}(\vec{E}_t) \subseteq \mathrm{support}(\vec{E}_{t'})$. We can rewrite~\eqref{fwefwt3464gregr} in terms of
\beq\label{fwefwt3464gre}
	\tau_t \subseteq \tau_{t'} \; \Leftrightarrow  \; \vec{E}_{t'}^T \vec{E}_{t} = \| \vec{E}_{t} \|_1.
\eeq
In NNMs, tag-vectors are usually not exactly binary and consequently,~\eqref{fwefwt3464gre} cannot be satisfied for any pair of tags. However, from a practical perspective, it makes still sense to ask whether most of the weight of $\vec{E}_t$ is contained in most of the weight of $\vec{E}_{t'}$. This motivates the relaxation~\cite{stark2015expressive}
\beq\label{few64ah}
	\tau_t \subseteq_{\varepsilon} \tau_{t'} \; :\Leftrightarrow  \; \vec{E}_{t'}^T \vec{E}_{t} \geq  (1 - \varepsilon) \| \vec{E}_{t} \|_1
\eeq
of~\eqref{fwefwt3464gre}. This discussion that led to~\eqref{few64ah} can be `quantized': after the embedding~\eqref{fekj45k653} of NNMs into quantum models, \eqref{few64ah} reads
\beq\label{fekjh6jmneff}
	\tau_t \subseteq_{\varepsilon} \tau_{t'} \; :\Leftrightarrow  \; \tr(E_{t} E_{t'}) \geq  (1 - \varepsilon) \| \vec{E}_{t} \|_1.
\eeq
We suggest presenting the collection of all pairwise relationships $\tau_t \subseteq_{\varepsilon} \tau_{t'}$ in terms of a directed graph $G = (V,E)$ where each vertex corresponds to a tag and where $t \rightarrow t'$ if and only if $\tau_t \subseteq_{\varepsilon} \tau_{t'}$. Visualizations of that type are \emph{important in practice} because it is often difficult to succinctly summarize datasets.

Of course, we could capture ``$\text{$\vec{p}_u^T\vec{E}_t \approx 1$ implies $\vec{p}_u^T\vec{E}_{t'} \approx 1$}$" more precisely through the following more accurate but computationally more expensive definition of $\tau_t \subseteq_{\varepsilon} \tau_{t'}$.\footnote{Computational considerations become relevant in the extreme multi-label limit~\cite{hsu2009multi,prabhu2014fastxml,zhang2011multi,ji2008extracting,agrawal2013multi}. In this multi-label limit, computational considerations are relevant because checking all the pairwise relations $\tau_t \subseteq_{\varepsilon} \tau_{t'}$ scales quadratically in the number of tags/labels. On the other hand, running a fixed number of iterations of alternating optimization to infer quantum models scales linearly in the number of users and items.} We check whether the semidefinite program
\beq\label{fewl435jkfgnr}
	\mathrm{find} \text{ $\rho$ s.t. $\tr(\rho E_t) \geq 1-\varepsilon/2$, $| \tr(\rho(E_{t'}-E_t)) | > \varepsilon/2$.}
\eeq
has a solution. If~\eqref{fewl435jkfgnr} has no solution and if 
\beq\label{fj43h5jhsd}
	\| E_t \|  =   \max\{\mathrm{spectrum}(E_t)\} \geq 1-\varepsilon/2
\eeq
then we declare membership $\tau_t \subseteq_{\varepsilon} \tau_{t'}$. That is because in this case we have that
\beq\begin{split}
	&\tr(\rho E_t) \geq 1-\varepsilon/2 \ \Rightarrow \ \\ 
	&\ \ \ \ \tr(\rho E_{t'}) = \tr(\rho (E_{t'}-E_t))+\tr(\rho E_t) \geq (-\varepsilon/2) +\tr(\rho E_t) \geq 1-\varepsilon.
\end{split}\eeq
Note that the condition `$\tr(\rho E_t) \geq 1-\varepsilon/2$' is important. If we did not have this constraint then we could almost always find $\rho$ with $| \tr(\rho(E_t-E_{t'})) | > \varepsilon/2$ (in the case of NNMs, simply choose $\vec{p}$ to be supported in supp($E_{t'}$) minus supp($E_t$)). We require~\eqref{fj43h5jhsd} because without this condition, we would have $\tau_t \subseteq_{\varepsilon} \tau_{t'}$ for all $t'$ whenever $E_t \approx 0$.

%%%%%
\section{Experiments}\label{sect:numerical.eval}

In this section we report our findings about the practical performance of quantum models in the context of recommendation. In the appendix we specify the configuration of all the algorithms used in the numerical experiments. The performance of previously known recommender systems was evaluated using the java library LibRec.\footnote{http://www.librec.net/} We evaluate the performance of quantum models in terms of \emph{mean-average-error} (MAE), \emph{Root-mean-squared-error} (RMSE) and \emph{recall}. The appendix of~\cite{stark2015expressive} contains a reminder of the definitions of MAE, RMSE and recall. Moreover, the appendix of~\cite{stark2015expressive} includes an argument for favoring MAE over RMSE when judging the performance of recommender systems. All results were computed on a desktop computer (4 cores; 16GB RAM) running Matlab calling SDPT3~\cite{toh1999sdpt3}. We evaluate the performance of quantum models on the MovieLens 100K and 1M dataset~\cite{miller2003movielens} so that we can compare our results with prior works~\cite{cremonesi2010performance}.

\emph{Top-N recommendation.} We select the training data $\Gamma_{\mathrm{train}}$ and test data $\Gamma_{\mathrm{test}}$ from the MovieLens 1M dataset as in~\cite{cremonesi2010performance}. This amounts to randomly choosing 1.4\% of the overall ratings as test set. As in~\cite{cremonesi2010performance}, $i$ is considered to be \emph{relevant} for user $u$ if $R_{ui} = 5$.

Figure~\ref{fig:recall.as.fn.of.iteration.CQ.allItems} compares quantum models with widely used recommender systems (see~\cite{cremonesi2010performance} for details; NNM60 refers to a 60-dimensional NNM~\cite{stark2015expressive}). We note that quantum models \emph{perform very well for small $N$} and there is a gap between NNMs and quantum models. This is of interest in applications because we do not want to present long lists of recommendations to users. Figure~\ref{fig:recall.as.fn.of.dim.Complex.Quant.and.fig:recall.as.fn.of.iteration.NNM} displays recall at 20 as function of the iteration of Algorithm~\ref{Alg:Constrained.least.squares.for.NPSD}. 

\begin{figure}[tbp]
\centering
\includegraphics[width=1\columnwidth]{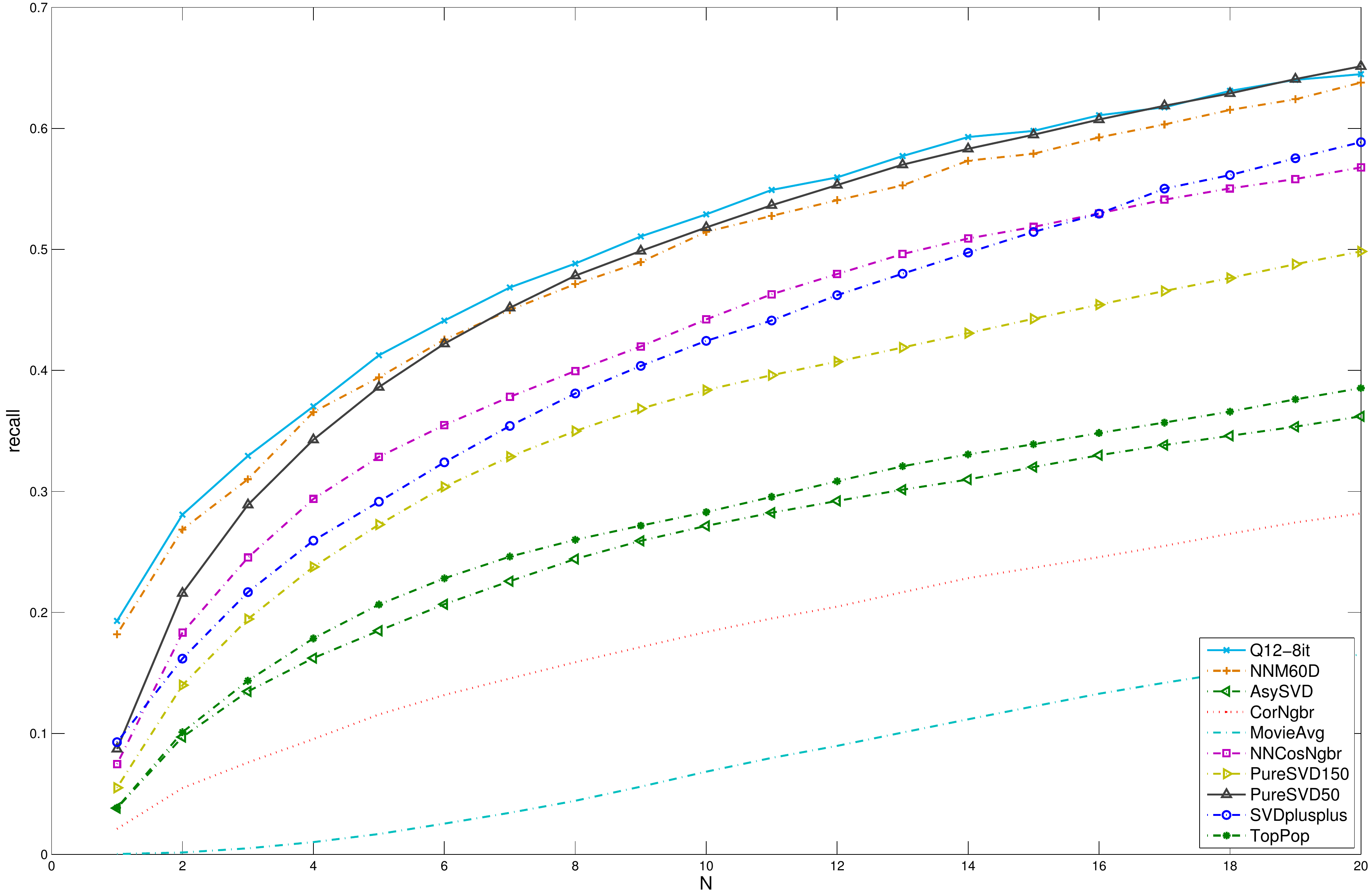}
\caption{Recall at $N$ (all items). Q12-8it is a 12-dimensional quantum model computed with 8 iterations of Algorithm~\ref{Alg:Constrained.least.squares.for.NPSD}.}
\label{fig:recall.as.fn.of.iteration.CQ.allItems}
\end{figure}

\emph{MAE and RMSE.} Table~\ref{Table.with.numerical.results} summarizes empirical results which were obtained via 5-fold crossvalidation (0.8 to 0.2 splitting of the data). We note that quantum models perform well in MAE, matching the performance of state-of-the-art recommender systems. Indeed quantum models outperform SVD++ on the MovieLens datasets. Here, quantum models, NNMs and SVD++ are optimized for small RMSE. Figure~\ref{fig:recall.as.fn.of.iteration.Complex.Quant} displays MAE as function of the iteration of Algorithm~\ref{Alg:Constrained.least.squares.for.NPSD}. We observe that the performance of quantum models (measured in MAE, RMSE) matches the performance of NNMs (at least as long as we use simple alternating optimization for the inference of quantum models).

\begin{figure}
\centering
\begin{minipage}[t]{.5\textwidth}
\centering
\vspace{0pt}
\includegraphics[width=\textwidth]{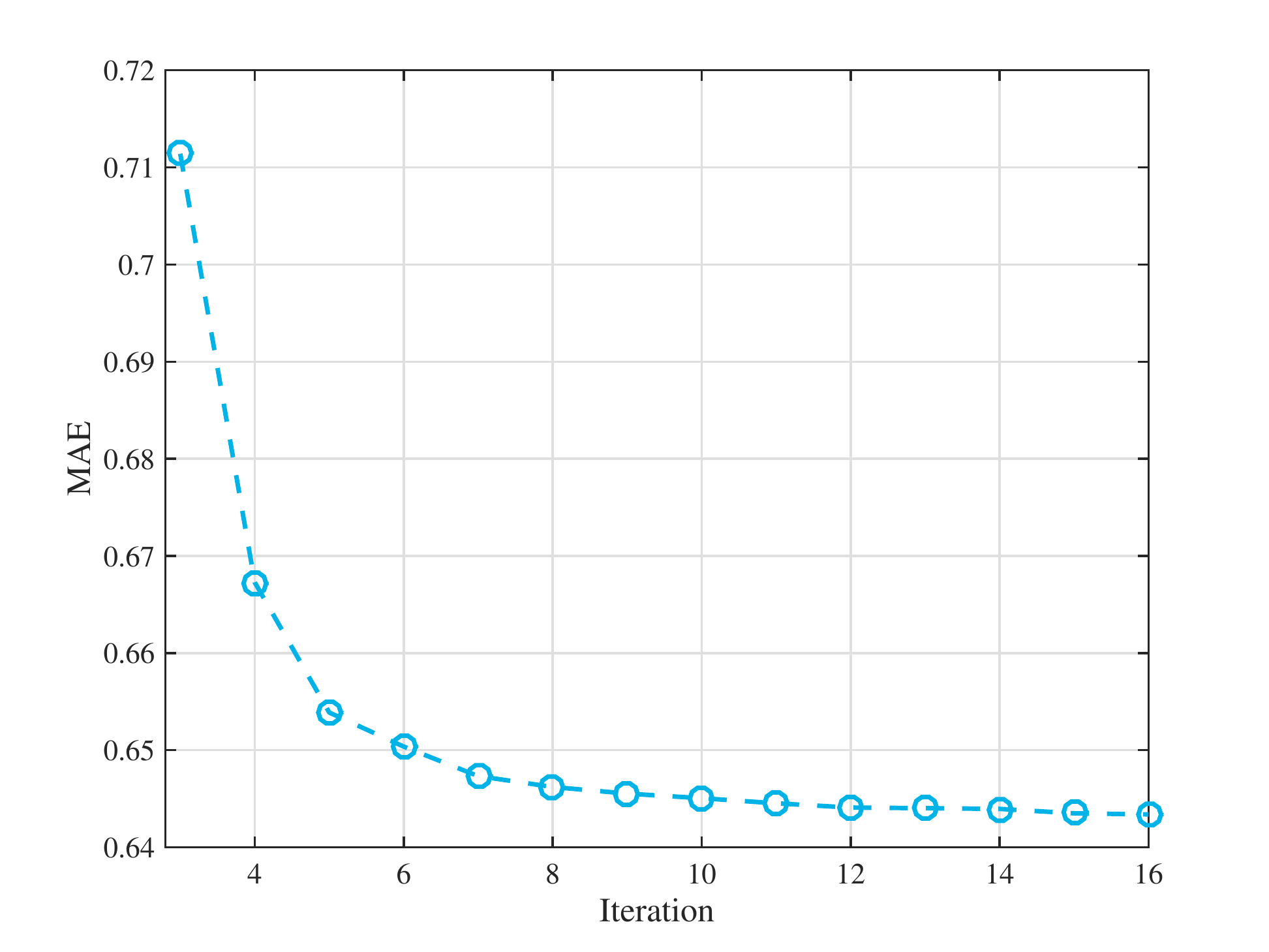}
\caption{MAE as a function of the iteration of Algorithm~\ref{Alg:Constrained.least.squares.for.NPSD} for $D = 3$ (MovieLens 1M)}
\label{fig:recall.as.fn.of.iteration.Complex.Quant}
\end{minipage}\hfill
\begin{minipage}[t]{.5\textwidth}
\centering
\vspace{0.6cm}
\captionof{table}{MAE and RMSE for the MovieLens 100K and 1M datasets.}
\label{Table.with.numerical.results}
\vspace{0.15cm}
\begin{tabular}{|l|c|c|c|c|}
\hline
								& \emph{100K},	& \emph{100K},	& \emph{1M},		& \emph{1M}, \\ 
								& MAE	& RMSE	& MAE	& RMSE	 \\ 
\hline
UserKNN~\cite{resnick1994grouplens}	& 0.74	& 0.94				& 0.70		& 0.91		 		\\
ItemKNN~\cite{rendle2009bpr}			& 0.72	& 0.92	& 0.69		& 0.88		 	 		\\
NMF~\cite{lee2001algorithms}			& 0.75	& 0.96		& 0.73		& 0.92		 		\\
SVD++~\cite{koren2008factorization}	& 0.72	& {\bf 0.91}	& 0.67		& {\bf 0.85}						\\
NNM	~\cite{stark2015expressive}		& {\bf 0.70}	& 0.98	& {\bf0.64} 	& 0.92	\\
Quantum							& {\bf 0.70}	& 0.99	& {\bf 0.64}		& 0.92 \\
\hline
\end{tabular}
\end{minipage}
\end{figure}

\begin{figure}
\centering
\begin{subfigure}{.5\textwidth}
  \centering
  \includegraphics[width=1\linewidth]{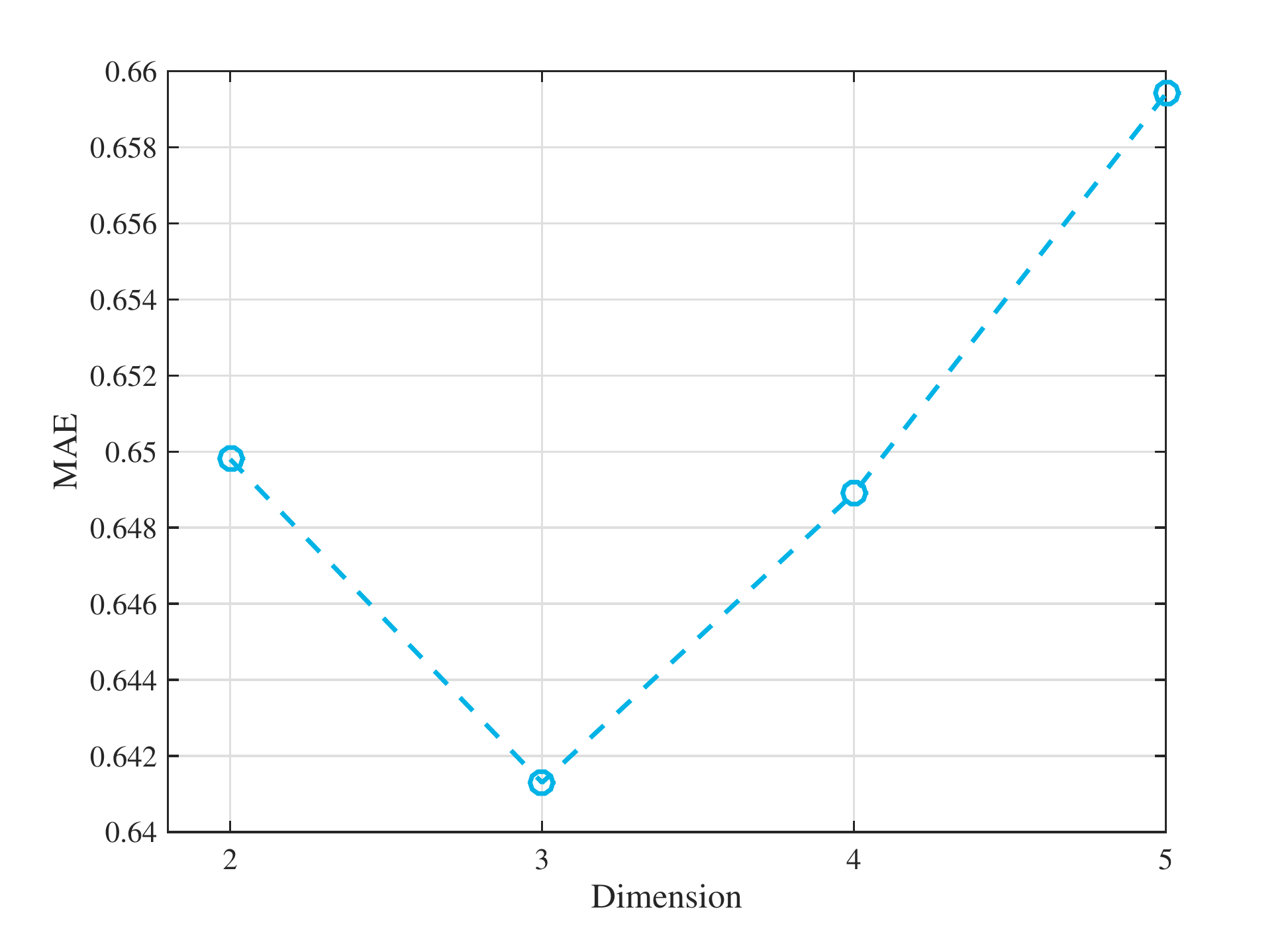}
\end{subfigure}%
\begin{subfigure}{.5\textwidth}
  \centering
  \includegraphics[width=1\linewidth]{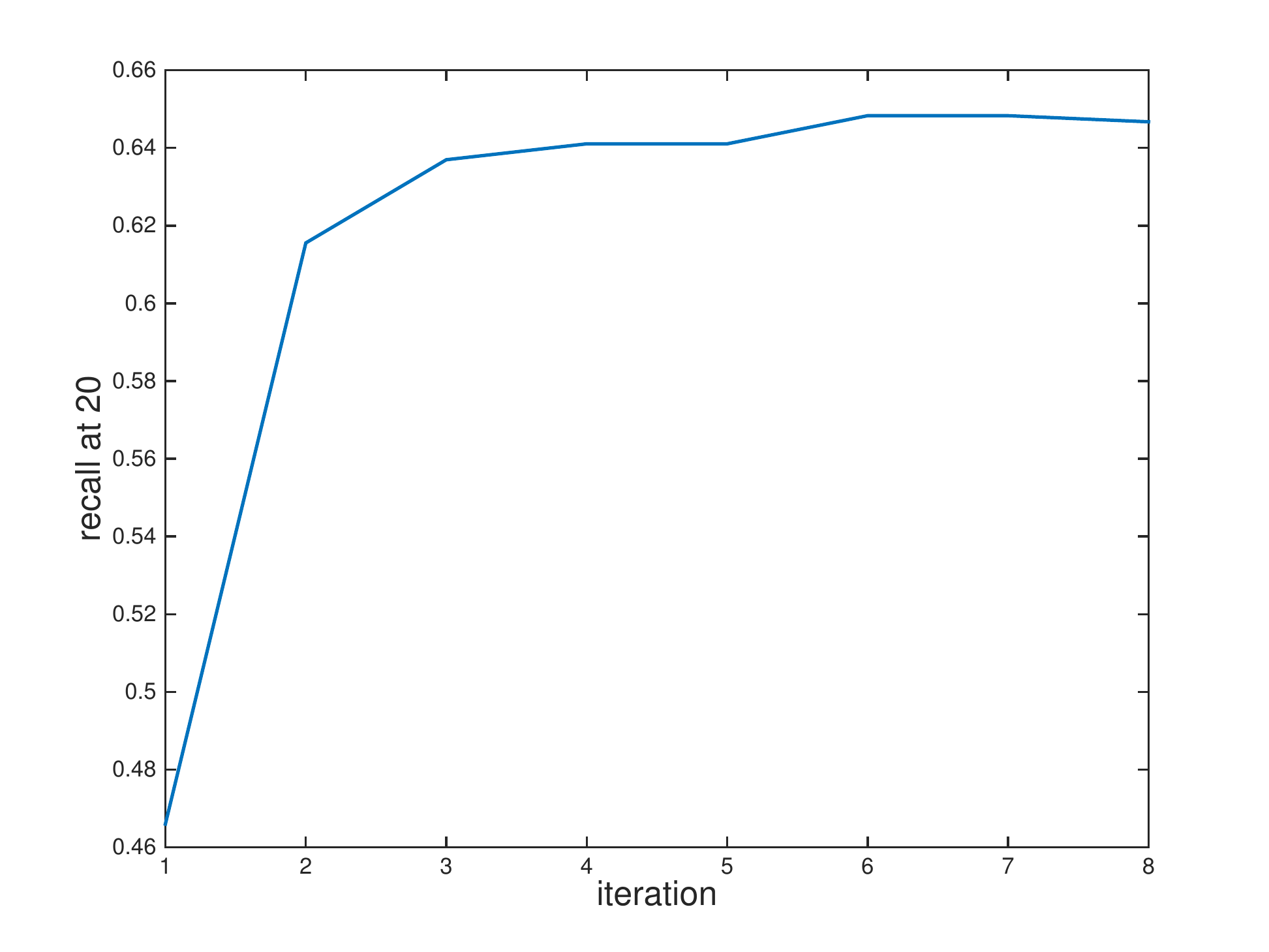}
\end{subfigure}
\caption{\emph{Left:} MAE as a function of the dimension $D$ (MovieLens 1M). \emph{Right:} Recall at 20 (all items) as function of iteration; $D = 12$ (MovieLens 1M).}
\label{fig:recall.as.fn.of.dim.Complex.Quant.and.fig:recall.as.fn.of.iteration.NNM}
\end{figure}

\emph{Hierarchical orderings.} In section~\ref{sect:hierarchical.ordering} we introduced a method to extract hierarchical orderings of properties of users and items. Here, we apply this method~\eqref{fekjh6jmneff} to quantum models computed for the MovieLens 1M dataset. The result is the directed graph in figure~\ref{ML1M_hierarchy_graph}. The vertices of the property graph correspond to movie genres like \emph{Action}, \emph{Comedy}, \emph{Drama}. To reduce the complexity of figure~\ref{ML1M_hierarchy_graph}, we excluded the vertices corresponding to the generes `Film-Noir' and `War' from the figure. Movies associated to these genres are rated highly by a majority of users and thus, all genres are connected to those genres: because of the selection bias, said genres must lie in the intersection of every other genre.

\begin{figure*}
\centering
\includegraphics[width=0.9\columnwidth]{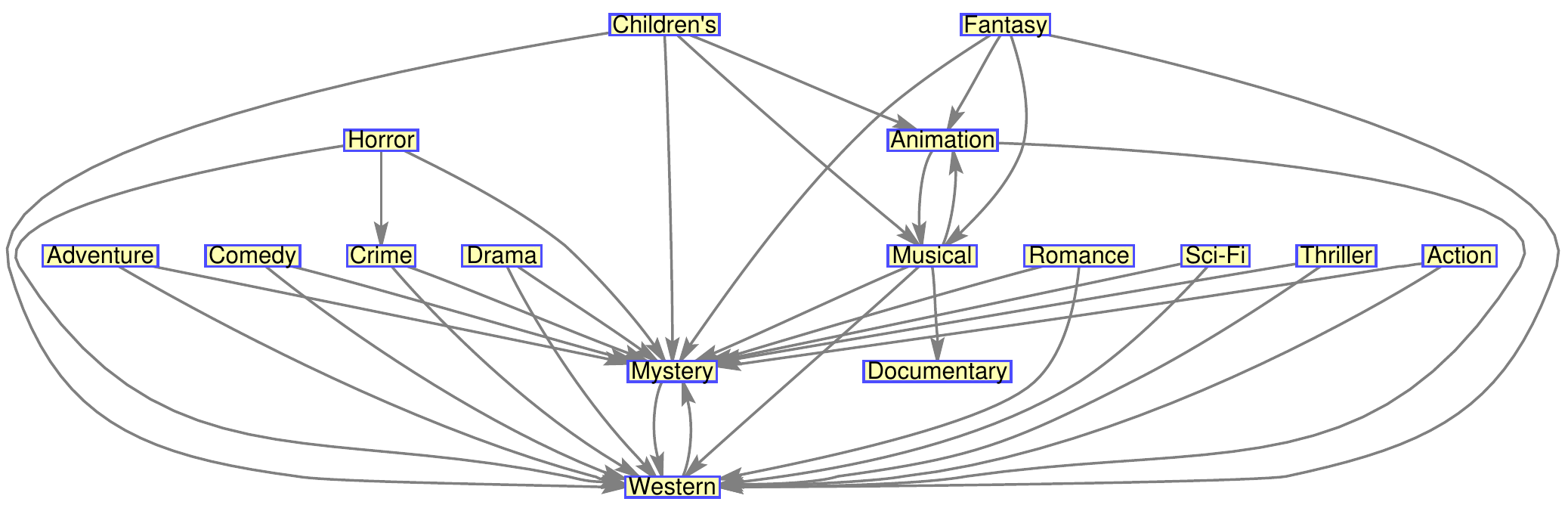}
\caption{Genre hierarchy at $\varepsilon = 1/3$ and $D = 3$. This is a hierarchy extracted from a quantum model fitting the MovieLens 1M dataset (cf.~\cite{stark2015expressive} for the comparison with NNMs).}
\label{ML1M_hierarchy_graph}
\end{figure*}

%%%%%
\section{Conclusion}\label{sect:conclusion}

Quantum models have been so successful in physics because of their \emph{interpretability} (allowing for \emph{conclusions} from user and item representations), their \emph{relation to probabilistic descriptions} and because of their \emph{high predictive power}. Quantum models inherit their interpretability from Kolmogorov factorizations (section~\ref{Sect:Normalized.nonnegative.models}). Here, we made use of this connection by adapting a method from~\cite{stark2015expressive} for the computation of hierarchical orderings of tags of users and items. Their relation to Kolmogorov factorizations makes quantum models more expressive than general matrix factorizations which are difficult to interpret; see section~3 of~\cite{stark2015top}.

On the other hand, as illustrated in figure~\ref{Fig:dim.interpretability.tradeoff}, high interpretability comes at the price of higher dimensionality of representations (e.g., user and item representations). Theorem~\ref{thm:gap.thm} makes this more precise: by Theorem~\ref{thm:gap.thm}, there exist datasets with the property that the dimension of the optimal approximate NNMs is exponential in the dimension of the optimal approximate quantum model. Consequently, there exist datasets for which we can compute approximate quantum models but for which we cannot compute approximate NNMs. Using a lower bound from~\cite{stark2014compressibility} and~\cite{lee2014some}, the paper~\cite{stark2014compressibility} showed an similar exponential separation between quantum models and general matrix factorizations (apply said lower bound to the example $E_{1z} = \vec{e}_z \vec{e}_z^T$, $Z=D$, from the introduction in~\cite{stark2014compressibility}).

\begin{figure}
\centering
% Define box and box title style
\tikzstyle{mybox} = [draw=black!70, fill=gray!10, very thick,
    rectangle, rounded corners, inner sep=10pt, inner ysep=20pt]
\tikzstyle{fancytitle} =[fill=black!70, text=white]

% Draw figure with parts 
\begin{tikzpicture}
       \node (a) at (0,-0.20)
         {
            	\begin{tikzpicture}
			\node [mybox] (box){%
    			\begin{minipage}{0.170\textwidth}
        				
				\vspace{-2em}
        				\begin{tikzpicture}[scale=1.0]
					% coordinate system
					\draw[->] (0,0) -- (2,0);
					\draw[->] (0,0) -- (0,2);
					\node at (-0.24,2.2) {$\mathbb{R}^d$};		
					\draw[dashed] (1.4,0) -- (0,1.4);		
					\draw[redish,thick,->] (0,0) -- (0.42,0.98);
					\node[redish] at (0.72,1.24) {$\vec{p}_{u}$};
					\draw[redish,thick,->] (0,0) -- (1.4,0.0);
					\node[redish] at (1.5,-0.34) {$\vec{E}_{iz}$};
				\end{tikzpicture}
				\vspace{-4em}				
        
    			\end{minipage}
			};
			\node[fancytitle, right=10pt] at (box.north west) {Kolmogorov};
		\end{tikzpicture}
         };
        \node (b) at (a.east) [anchor=west,yshift=0cm,xshift=0cm]
         {
            	\begin{tikzpicture}
			\node [mybox] (box){%
    			\begin{minipage}{0.170\textwidth}
        				
				\vspace{-2em}
        				\begin{tikzpicture}[scale=0.995]
					% coordinate system
					\draw[->] (0,0) -- (2,0);
					\draw[->] (0,0) -- (0,2);
					\node at (-0.8,2.2) {$\mathbb{R}^d$};		
					\draw[dashed] (1.4,0) -- (0,1.4);		
					\draw[redish,thick,->] (0,0) -- (0.42,0.98);
					\node[redish] at (0.16,1.24) {$\vec{p}_{u}$};
					\draw[redish,thick,->] (0,0) -- (1.3,0.8);
					\node[redish] at (1.05,0.8) {$\vec{E}_{iz}$};
					\node at (0,2.1) {};					
					\node at (0,-0.55) {};
				\end{tikzpicture}
        				\vspace{-4em}
    			\end{minipage}
			};
			\node[fancytitle, right=10pt] at (box.north west) {NNM};
		\end{tikzpicture}
         };
         \node (c) at (b.east) [anchor=west,yshift=0cm,xshift=0cm]
         {
            	\begin{tikzpicture}
			\node [mybox] (box){%
    			\begin{minipage}{0.170\textwidth}
        				
				\vspace{-1.9em}
        				\begin{tikzpicture}[scale=0.93]
					% cone
					\draw[fill=gray!30,rotate around={-45:(1.1,1.1)}] (1.1,1.1) ellipse (0.9cm and 0.5cm);
					\draw[] (0,0) -- (2.325,0.555);
					\draw[] (0,0) -- (0.57,2.355);	
	
					% states and measurements
					\draw[redish,thick,->] (0,0) -- (1.5,0.65);
					\node[redish] at (1.05,0.9) {$\rho_{u}$};
					\draw[redish,thick,->] (0,0) -- (1.55,1.8);
					\node[redish] at (1.25,1.8) {$E_{iz}$};
					
					\node at (-0.2,2.6) {$S^+(\mathbb{C}^d)$};				
				\end{tikzpicture}
        				\vspace{-0.15em}
    			\end{minipage}
			};
			\node[fancytitle, right=10pt] at (box.north west) {quant model};
		\end{tikzpicture}
         };
         \node (d) at (c.east) [anchor=west,yshift=0cm,xshift=0cm]
         {
            	\begin{tikzpicture}
			\node [mybox] (box){%
    			\begin{minipage}{0.170\textwidth}
        				
				\vspace{-2.9em}
        				\begin{tikzpicture}[scale=1.22]
					% coordinate system
					\draw[->] (-1,0) -- (1,0);
					\draw[->] (0,-1) -- (0,1);
					\node at (-0.7,1.1) {$\mathbb{R}^d$};
	
					% vectors
					\draw[redish,->,thick] (0,0) -- (0.8,0.2);	
					\draw[redish,->,thick] (0,0) -- (0.6,-0.8);		
					\node[redish] at (0.56,0.35) {$\vec{v}_u$};	
					\node[redish] at (0.4,-0.8) {$\vec{r}_i$};
					
					\node at (0,1.46) {};
					\node at (0,-0.54) {};	
				\end{tikzpicture}
				\vspace{-3.17em}
        
    			\end{minipage}
			};
			\node[fancytitle, right=10pt] at (box.north west) {unreg matrix fact};
		\end{tikzpicture}
         };
         
         \draw[>=stealth,->,thick] (0.88,-3) -- (13.15,-3);
         \node at (-0.5,-3) {\emph{lower dimension}};
         \draw[>=stealth,->,thick,] (9.7,-3.35) -- (-1.72,-3.35);
         \node at (11.5,-3.35) {\emph{higher interpretability}};
         
\end{tikzpicture}
\caption{Illustration of the dimension-interpretability tradeoff. From left to right: Kolmogorov factorization, NNM, quantum model, unregularized matrix factorization (low-rank matrix completion; $R = AB$, $A \in \AR^{U \times k}$, $B \in \AR^{k \times I}$).}
\label{Fig:dim.interpretability.tradeoff}
\end{figure}
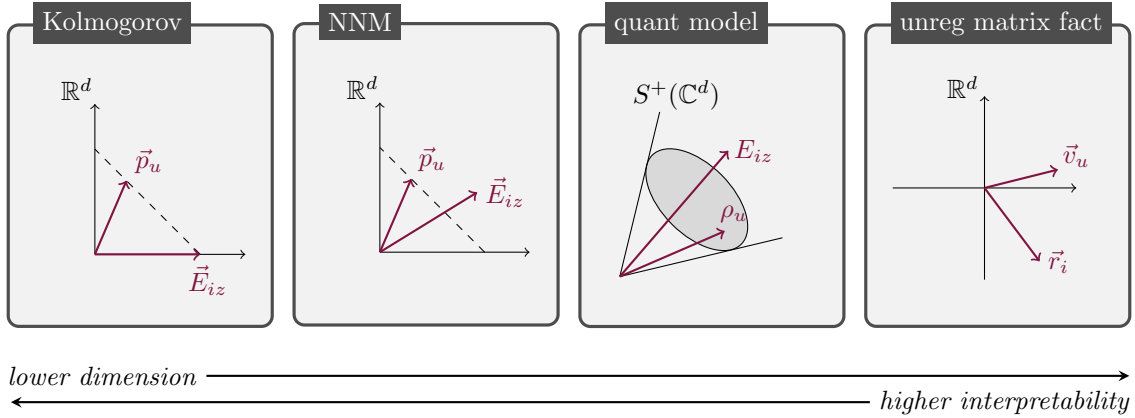

Given a dataset that does not allow for the computation of NNMs, we are confronted with the challenge to find other models with high expressiveness which can be fitted to larger classes of datasets. Quantum models meet this challenge by realizing a tradeoff between high interpretability and low dimensionality of representations. To illustrate the flexibility of quantum models, we derived Theorem~\ref{thm:fkerjtk4jk} which states that sparse datasets admit low-dimensional quantum models. In physics, the relation between Kolmogorov factorizations and quantum models is usually established through `quantization rules' or deformation quantization~\cite{bayen1977quantum}. Here we provide a novel relation between NNMs and quantum models by showing that a large subset of quantum state space and measurement space can be interpreted as compression of pseudo sparse NNMs; see Theorem~\ref{cor:fj345ewklnnr}. Moreover, we show that quantum models can be a computationally convenient relaxation of NNMs because there exist datasets for which we can recover the optimal NNM from the optimal quantum model; see Lemma~\ref{lem:fwerg435e}.

All of the above considerations would be of more theoretical than practical nature if we were not able to fit approximate quantum models to datasets. Therefore, we used a simple alternating optimization scheme to compute quantum models for real datasets, and we demonstrated the power of quantum models in the context of item recommendation. More precisely, we compute quantum models for MovieLens datasets and observed that quantum-based recommender systems can compete with state-of-the-art methods like SVD++ and PureSVD. 

Due to their value in recommendation, quantum models provide a new link between data science (item recommendation), information theory (communication complexity), mathematical programming (positive semidefinite factorizations) and physics (quantum models).

%%%%%
\section{Acknowledgment}

I thank David Gross, Aram Harrow, Robin Kothari, Patrick Pletscher, Renato Renner, L{\'i}dia del Rio and Sharon Wulff for fruitful discussions. In particular, I thank Robin Kothari for pointing out reference~\cite{Gavinsky} about one-way communication complexity of partial functions. I acknowledge funding by the ARO grant Contract Number W911NF-12-0486, and I acknowledge funding by the Swiss National Science Foundation through a fellowship. This work is preprint MIT-CTP/4761.

\appendix

%%%%%
\section{Proof of Theorem~\ref{cor:fj345ewklnnr}}

Theorem~\ref{cor:fj345ewklnnr} turns out to be a corollary of the following Theorem~\ref{Thm:fjewklnnr} from~\cite{stark2014compressibility}.

\begin{theorem}[see~\cite{stark2014compressibility}]\label{Thm:fjewklnnr}
	Let $J = U + IZ$, let $\varepsilon$ satisfy $\frac{4}{\sqrt{JD}} \leq \varepsilon \leq \frac{1}{2}$ and let $\bigl((\rho_{u})_{u}, (E_{iz})_{iz} \bigr)$ be a $D$-dimensional quantum model. Assume
	\begin{itemize}
				\item[]	$\bigl((\rho_{x})_{x}, (E_{yz})_{yz} \bigr)$ is \emph{pseudo low-rank}, and			
				\item[]	$d = \mathcal{O}\Bigl( \frac{1}{\varepsilon^2} \; \ln\bigl( 4JD \bigr) \Bigr)$,
	\end{itemize}
	Then, there exists a $d$-dimensional quantum model $\bigl( (\rho'_{u})_{u}$ $(E'_{iz})_{iz} \bigr)$ satisfying
	\[
		\bigl| \tr(\rho_{u}E_{iz}) - \tr(\rho'_{u}E'_{iz}) \bigr| \leq \mathcal{O}(\varepsilon)
	\]
	for all $u,i,z$.
\end{theorem}

\begin{proof}[Proof of Theorem~\ref{cor:fj345ewklnnr}]
The map  
\beq\begin{split}\label{dwj32i5jfnej}
	\vec{p}_u &\mapsto \rho_{u} := \mathrm{diag}(\vec{p}_u) = \sum_{j=1}^{D} (\vec{p}_{u})_{j} \; \vec{e}_j \vec{e}^T_j\\
	\vec{E}_{iz} &\mapsto E_{iz} := \mathrm{diag}(\vec{E}_{iz}) = \sum_{j=1}^{D} (\vec{E}_{iz})_{j} \; \vec{e}_j \vec{e}^T_j.
\end{split}\eeq
realizes an embedding of NNMs into quantum models. Applying Theorem~\ref{Thm:fjewklnnr} to that embedding of NNMs suffices to prove the claim.	
\end{proof}

%%%%%
\section{Proof of Theorem~\ref{thm:fkerjtk4jk}}

	Let $N$ be the maximal number of times an item has been rated. Hence, $N \sim \mathcal{O}(1)$ in $U$. Recall that $\Gamma \subseteq [U] \times [I]$ marks the provided ratings. Thus, on the complement $\Gamma^c$, the ratings are unspecified. Set $\Gamma_i := \{ u \in [U] | (u,i) \in \Gamma \}$. We observe that 
	\beq\begin{split}\label{fw4k5jhfr}
		\rho_u	&= \vec{e}_u \vec{e}_u^T, \\
		E_{iz}	&= \Bigl( \sum_{u \in \Gamma_{i}}  \delta_{z,R_{ui}} \; \vec{e}_{u} \vec{e}_{u}^T \Bigr) + \Bigl( \sum_{u \in \Gamma^c_{i}}  \delta_{z,1} \; \vec{e}_{u} \vec{e}_{u}^T \Bigr)
	\end{split}\eeq
	is a $U$-dimensional quantum model that fits the known data perfectly. Indeed, for every $(u,i) \in \Gamma$,
	\beq\begin{split}\nn
		\tr(\rho_u E_{iz}) 
		&=	\Bigl( \sum_{u' \in \Gamma_{i}} \delta_{z,R_{u'i}} \; \delta_{uu'} \Bigl) + \Bigl( \sum_{u' \in \Gamma^c_{i}} \delta_{z,1} \; \delta_{uu'} \Bigl) \\
		&=	\delta_{z,R_{ui}}.
	\end{split}\eeq
	Note that for all $z \geq 2$, 
	\beq\begin{split}\nn
		&\rank( E_{iz} )\\
		&\leq \rank\Bigl( \sum_{u \in \Gamma_{i}}  \delta_{z,R_{ui}} \; \vec{e}_{u} \vec{e}_{u}^T \Bigr) + \rank\Bigl( \sum_{u \in \Gamma^c_{i}}  \delta_{z,1} \; \vec{e}_{u} \vec{e}_{u}^T \Bigr) \\
		&\leq N \sim \mathcal{O}(1)
	\end{split}\eeq
	because $\rank(A+B) \leq \rank(A) + \rank(B)$ for all matrices $A,B$. Hence, the quantum model~\eqref{fw4k5jhfr} is \emph{pseudo-low rank}. Therefore, applying Theorem~\ref{Thm:fjewklnnr} is sufficient to conclude the proof.

%%%%%
\section{Proof of Theorem~\ref{thm:gap.thm}}

	By the embedding~\eqref{fekj45k653}, $d_{\mathrm{Q}} \leq d_{\mathrm{NNM}}$ always. To prove the main claim of the theorem, we relate $d_{\mathrm{NNM}}$ and $d_{\mathrm{Q}}$ to the one-way communication complexity of Boolean functions. Before describing this relation, we remind the reader of randomized (quantum) bounded error one-way communication complexity of Boolean functions.
	
	Let $f: \{ 0,1 \}^n \times \{ 0,1 \}^m \rightarrow \{0,1\}$ be a Boolean function. Assume Alice receives an $n$-bit string $x$ and Bob receives and $m$-bit string $y$. Both Alice and Bob are given the possibility to sample private random variables. In one-way randomized communication, Alice is allowed to send a random variable $W$ with some distribution $\vec{p}_x \in \AR^d_+$ of her choice to Bob. Upon receiving this message, Bob tries to guess $f(x,y)$ by evaluating a random variable $V_y \in \{0,1\}$ on $W$ and on his private randomness. If $V_y = 0$ then Bob claims $f(x,y) = 0$ and otherwise, Bob claims $f(x,y) = 1$.  The one-way bounded error communication complexity of $f$ is the minimal number of bits Alice needs to send to Bob so that Bob's guess equals $f(x,y)$ with probability  $\geq 1-\varepsilon$ where $\varepsilon \in [0,1/2)$ (usually, $\varepsilon = 1/3$). By a technique called amplification, you can show that changing the value of $\varepsilon \in [0,1/2)$ changes the communication complexity only by a constant factor. Communication complexity of Boolean functions has been studied for both total and partial functions. In the remainder, we are focusing on partial functions where $f$ is only defined on some subset $\Gamma \subseteq \{ 0,1 \}^n \times \{ 0,1 \}^m$ of all inputs.
	
	The definition of \emph{quantum} communication complexity is almost identical to the definition of communication complexity: instead of sending a random variable to Bob, Alice sends a quantum system in some state $\rho_x$ to Bob who performs a measurement $(E_y,I - E_y)$ on $\rho_x$ and some auxiliary quantum system (instead of sampling a random variable $V_y$). The complexity of the communication protocol is set equal the logarithm of the dimension of $\rho_x$.
	
The following Lemma~\ref{lem:fiewjhfijoi} (used in~\cite{stark2014compressibility} to compress communication protocols) states approximate equality between 
\begin{itemize}
\item the optimal dimension of approximate NNMs fitting a partially known rating matrix $R_{\Gamma}$ and
\item the bounded error one-way communication complexity of a partial function whose communication matrix is $R_{\Gamma}$.
\end{itemize}
The analogous result about quantum communication complexity is stated in Lemma~\ref{lem:fekwrhfkj4eht}.

\begin{lemma}\label{lem:fiewjhfijoi}
	Let $\varepsilon \in [0,1/2)$, $f: \{ 0,1 \}^n \times \{ 0,1 \}^m \supseteq \Gamma \rightarrow \{0,1\}$ be a partial function, and let $A$ be the communication matrix of $f$. We denote by $R^1(f)$ the $\varepsilon$-bounded error randomized one-way communication complexity of $f$. Let $d^*$ denote the dimension of the lowest-dimensional NNM $\vec{p}_x, \vec{E}_{y}$ satisfying 
	\beq\begin{split}\label{fewjfkj3k45j}
		| \vec{p}_x^T \vec{E}_{y} - A_{xy} | \left\{
\begin{array}{ll}
  \geq 1-\varepsilon,	&\text{if $A_{xy} = 1$},   \\
  \leq  \varepsilon,	&\text{if $A_{xy} = 0$}
\end{array} \right.
	\end{split}\eeq
	for all $(x,y) \in \Gamma$. Then, $R^1(f)-1 \leq \log_2(d^*) \leq R^1(f)$.
\end{lemma}

\begin{proof}

We begin with the proof of the claim $R^1(f)-1 \leq \log_2(d^*)$. Let $\vec{p}_x, \vec{E}_y$ be a $d$-dimensional approximate NNM. We claim that we can use the model $\vec{p}_x, \vec{E}_y$ to construct a communication protocol with cost $\leq \log_2(d) +1$. 

By definition of NNMs, $\vec{E}_y \in [0,1]^d = \mathrm{conv}(\{0,1\}^d)$. Thus, by Caratheodory's theorem, there exist a probability distribution $\vec{q}_y \in \Delta_{d+1}$ such that $\vec{E}_y = \sum_{j=1}^{d+1} q_{yj} \vec{F}^y_{j}$ for some $\vec{F}^y_{j} \in \{0,1\}^d$. We need to construct a communication protocol with cost $\leq \log_2(d) +1$. We suggest the following protocol.

Upon receiving input $x$, Alice sends a sample of the random variable $W$ with alphabet $[d]$ and distribution $\vec{p}_x \in \Delta_d$ to Bob. Upon receiving $W$ from Alice, Bob first samples his private $d+1$-sided dice (private randomness) whose outcomes are distributed according $\vec{q}_y$; we call the sampled outcome $j'$. Then, Bob claims $f(x,y) = 1$ if $W \in \mathrm{supp}(F^y_{j'})$. Otherwise, Bob claims $f(x,y) = 0$. 

To check that this protocol proves the claim $R^1(f)-1 \leq \log_2(d^*)$, we need to verify that it has high enough winning probability. Assume $f(x,y) = 1$. Then, by assumption on the considered NNM,
\beq\begin{split}
	1-\varepsilon
	&\leq		\vec{p}_x^T \vec{E}_y 
	=		\vec{p}_x^T \Bigl( \sum_{j=1}^{d+1} q_{yj} \vec{F}^y_{j} \Bigr) 
	= 		\sum_{j=1}^{d+1} \vec{q}_y^T \otimes \vec{p}_x^T \; \vec{e}_j \otimes \vec{F}^y_{j}
	= 		\sum_{j=1}^{d+1} \mathbb{P}[ j' = j, W \in \mathrm{supp}(\vec{F}^y_{j}) ]
\end{split}\eeq
The marginal distribution on the rhs is the success probability of the proposed communication protocol. Said protocol requires $\lceil \log_2(d) \rceil$ bits of information and therefore, $R^1(f) \leq \log_2(d) + 1$.

Next we prove the remaining claim $\log_2(d^*) \leq R^1(f)$. To prove this claim it suffices to show that starting from a communication protocol with cost $c$ we can construct a valid $2^c$-dimensional NNM. Assume that there exists a valid communication protocol with cost $c$. Hence, 
\begin{itemize}
\item		there exists a random variable $W$ (i.e., the random variable sent to Bob) with distribution $\vec{p}_x \in \Delta_{2^c}$ and 
\item		there exists a random variable $B_y: \Omega \times \Omega_{\text{private}} \rightarrow \{0,1\}$
\end{itemize}
such that $\mathbb{P}[B_y=1] \geq 1- \varepsilon$ if $f(x,y) = 1$ and $\mathbb{P}[B_y=1] \leq \varepsilon$ if $f(x,y) = 0$. Here, $\Omega = \{ \omega_1,...,\omega_{2^c} \}$ is the sample space corresponding to Alice's message $W$ and $\Omega_{\text{private}} = \{ \sigma_1,...,\sigma_{2^r} \}$ is the sample space of Bob's private randomness. We define the $2^{c+r}$-dimensional vector $\vec{E}_y = \sum_{ij}  (\vec{E}_y)_{ij} \, \vec{e}_i \otimes \vec{f}_j \in \{0,1\}^{2^c} \otimes \{0,1\}^{2^r}$ by
\beq
	(\vec{E}_y)_{ij} = \left\{ 	\begin{array}{ccc}  	1, &\text{ if $(\omega_i,\sigma_j) \in B^{-1}(1)$}   \\  
										0, &\text{ if $(\omega_i,\sigma_j) \in B^{-1}(0)$} .  
						\end{array}\right.
\eeq
We denote by $\vec{q}_y \in \Delta_{2^r}$ the distribution of Bob's randomness. By assumption, $\vec{p}_x, B_y$ is a valid communication protocol. Therefore, if $f(x,y) = 1$,
\beq\begin{split}
	1-\varepsilon
	&\leq		\mathbb{P}_{\vec{p}_x \otimes \vec{q}_y}[B_y = 1] 
	= 		\vec{p}_x^T \otimes \vec{q}_y^T \, E_y \\
	&= 		\vec{p}_x^T \otimes \vec{q}_y^T \, \Bigl( \sum_{i=1}^{2^c} \sum_{j=1}^{2^r} (\vec{E}_y)_{ij} \, \vec{e}_i \otimes \vec{f}_j \Bigr)
	= 		\sum_{i=1}^{2^c} \sum_{j=1}^{2^r} (\vec{E}_y)_{ij} \, (\vec{p}_x)_i (\vec{q}_y)_j \\
	&= 		\vec{p}_x^T \, \vec{E}'_y
\end{split}\eeq
where $\vec{E}'_y$ is defined by $(\vec{E}'_y)_i := \sum_{j=1}^{2^r} (\vec{E}_y)_{ij} (\vec{q}_y)_j$. Note that $\vec{E}'_y \in [0,1]^{2^c}$. Hence, $\vec{p}_x, \vec{E}'_y$ forms a valid $2^c$-dimensional approximate NNM. When setting $c = R^1(f)$ we arrive at the conclusion that $d^* \leq 2^{R^1(f)}$.

\end{proof}

The following Lemma~\ref{lem:fekwrhfkj4eht} about one-way quantum communication complexity is proven analogously\footnote{We simply need to replace bits with qubits and indicator vectors with projectors.} to Lemma~\ref{lem:fiewjhfijoi}.

\begin{lemma}\label{lem:fekwrhfkj4eht}
	Let $\varepsilon,f,\Gamma$ and $A$ be defined as in Lemma~\ref{lem:fiewjhfijoi}. We denote by $Q^1(f)$ the $\varepsilon$-bounded error \emph{quantum} one-way communication complexity of $f$. Let $d^*$ denote the dimension of the lowest-dimensional quantum model $\rho_x, E_{y}$ satisfying 
	\beq\begin{split}\label{fewjfkj3k45j}
		| \tr(\rho_x E_y) - A_{xy} | \left\{
\begin{array}{ll}
  \geq 1-\varepsilon,	&\text{if $A_{xy} = 1$},   \\
  \leq  \varepsilon,	&\text{if $A_{xy} = 0$}
\end{array} \right.
	\end{split}\eeq
	for all $(x,y) \in \Gamma$. Then, $Q^1(f)-1 \leq \log_2(d^*) \leq Q^1(f)$.
\end{lemma}

The following Theorem~\ref{thm:fjewhfjh} bounds the communication complexity of the so called $\alpha$-Partial Matching problem~\cite{Gavinsky}.

\begin{theorem}[see~\cite{Gavinsky}]\label{thm:fjewhfjh}
	Let $\alpha \in (0,1/4]$. The randomized bounded error one-way communication complexity of the $\alpha$-Partial Matching problem is $\Theta(\sqrt{n/\alpha})$ while the quantum bounded error one-way communication complexity of the $\alpha$-Partial Matching problem is $\mathcal{O}(\log(n)/\alpha)$.
\end{theorem}

Let $f: \Gamma \rightarrow \{0,1\}$ denote the partial function associated to the $\alpha$-Partial Matching problem, and let $A_{\Gamma}$ denote the corresponding (partially specified) communication matrix. Set $R_{\Gamma} = A_{\Gamma}$ and recall the definitions~\eqref{fekjlk5jfe} and~\eqref{fekjlk5j}. Then, by Theorem~\ref{thm:fjewhfjh} and Lemma~\ref{lem:fiewjhfijoi},
\[
	\log_2(d^*_{\mathrm{NNM}}) = \Theta(\sqrt{n/\alpha}).
\]
By Theorem~\ref{thm:fjewhfjh} and Lemma~\ref{lem:fekwrhfkj4eht},
\[
	\log_2(d^*_{\mathrm{Q}}) = \mathcal{O}(\log(n)/\alpha)
\]
for any $\alpha \in (0,1/4]$. This concludes the proof of Theorem~\ref{thm:gap.thm}.

%%%%%
\section{Proof of Lemma~\ref{lem:fwerg435e}}

Recall that $U,I,Z$ denote the number of users, items and possible ratings, respectively. Assume $U = I Z$ and let $\bigl( (\vec{p}_u)_u, (\vec{E}_{iz})_iz \bigr)$ be any $Z$-dimensional NNM with the property that
\[
	\mathbb{P}_{u}[ \hat{E}_{i} = z ] = \vec{p}_u^T \vec{E}_{iz}
\]
satisfies the following:
\begin{itemize}
\item		For each $(i,z) \in [I] \times [Z]$ there exists a user $u_{iz}$ such that $\mathbb{P}_{u_{iz}}[ \hat{E}_{i} = z ] = 1$.
\item		For each user $u \in [U]$ there exists $i_u \in [I]$ and $z_u \in [Z]$ such that $\mathbb{P}_{u}[ \hat{E}_{i_u} = z_u ]$.
\end{itemize}

These conditions define the class of NNMs referred to in Lemma~1. Let $\bigl( (\rho_{u})_{u}, (E_{iz})_{iz} \bigr)$ denote a $D$-dimensional quantum model satisfying $\tr( \rho_{u} E_{iz} ) = \vec{p}_{u}^T \vec{E}_{iz}$. By Cauchy-Schwarz,
\beq\label{fennsjher}
	1 = \tr( \rho_{u_{iz}} E_{iz} ) \leq \| \rho_{u_{iz}} \|_{F} \| E_{iz} \|_{F}
\eeq
implying $\| E_{iz} \|_{F} \geq 1$ because 
\beq\label{jhjefhJHbd}
	\| \rho_{u} \|_{F} \leq 1
\eeq 
always.\footnote{Use $\| \rho \|_{F} = \| O \rho O^T \|_{F}$ for every unitary matrix $O$ and the fact that $\| \vec{p} \|_{2} \leq 1$ whenever $\vec{p} \in \Delta$.} By $D = Z$ and $\sum_z E_{iz} = I$,
\beq\label{fejhjsehrj45}
\begin{split}
	Z &= \| I \|_{F}^{2} = \Bigl\| \sum_{z=1}^Z E_{iz}  \Bigr\|_{F}^2 \\ &= \Bigl( \sum_{z=1}^Z \| E_{iz} \|_{F}^2 \Bigr) + \Bigl( \sum_{z \neq z'} \tr( E_{iz} E_{iz'} )  \Bigr) \geq \sum_{z=1}^Z \| E_{iz} \|_{F}^2 
\end{split}
\eeq
because $\tr(AB) \geq 0$ if $A,B$ positive semidefinite. Assume $w \in [Z]$ was such that $\| E_{yw} \|_{F} > 1$. Then, by $\| E_{iz} \|_{F} \geq 1$ for all $i,z$,
\beq\label{fejhjsehrj45fwefe}
\begin{split}
	Z &\geq \sum_{z=1}^Z \| E_{iz} \|_{F}^2 \geq \| E_{yw} \|_{F} + (Z-1) \min_{z} \| E_{yz} \|_{F}^2 \\ &> Z \min_{z} \| E_{yz} \|_{F}^2 \geq Z.
\end{split}
\eeq
Therefore, $\| E_{yw} \|_{F} > 1$ is impossible. By~$\| E_{iz} \|_{F} \geq 1$, we conclude that $\| E_{iz} \|_{F} = 1$ for all $i,z$. By~\eqref{fennsjher} and~\eqref{jhjefhJHbd}, $\| \rho_{u_{iz}} \|_{F}= 1$. Let $\vec{s} \in \AR^D_+$ denote the vector of eigenvalues of $\rho_{u_{iz}}$ (ordered descendingly) so that
\[
	1 = \| \rho_{u_{iz}} \|_{F} = \| \vec{s} \|_{2}.
\]
Since $\| \vec{s} \|_{1} = 1$ this is only possible if $\vec{s} = (1,0,...,0)^T$, i.e., there exists $\vec{v}_{u_{iz}} \in \mathbb{C}^D$ ($\| \vec{v}_{u_{iz}} \|_{2} = 1$) such that $\rho_{u_{iz}} = \vec{v}_{u_{iz}} \vec{v}_{u_{iz}}^T$. Furthermore, Cauchy-Schwarz inequality~\eqref{fennsjher} is satisfied with equality. This happens if and only if there exists $\kappa \in \{ \pm 1 \}$ with $\rho_{u_{iz}} = \kappa E_{iz}$. Since both matrices $\rho_{u_{iz}}$ and  $E_{iz}$ are positive semidefinite, the alternative $\kappa = -1$ can be ruled out. Consequently, 
\beq\label{fejhw4hnfn4m5}
	E_{iz} = \rho_{u_{iz}} = \vec{v}_{u_{iz}} \vec{v}_{u_{iz}}^T
\eeq
for all $i,z$. By the rightmost inequality in~\eqref{fejhjsehrj45}, $\sum_{z \neq z'} \tr( E_{iz} E_{iz'}) = 0$. It follows that for every $z \neq z'$, $\tr( E_{iz} E_{iz'}) = \bigr( \vec{v}_{u_{iz}} \vec{v}_{u_{iz'}}^T \bigr)^2 = 0$, and we conclude that 
\beq\label{fewt4eg5}
	\vec{v}_{u_{iz}}^T \vec{v}_{u_{iz'}} = \delta_{zz'}.
\eeq

\begin{table}[!htbp]
\caption{Configurations of SVD++ \cite{koren2008factorization}, NMF \cite{lee2001algorithms}, UserKNN~\cite{resnick1994grouplens}, ItemKNN \cite{rendle2009bpr}, NNM~\cite{stark2015expressive}.}
\label{wefk4354}
 \centering
 \begin{tabular}{|c|l|l|}
 \hline
 & \multicolumn{1}{c|}{dataset} & \multicolumn{1}{c|}{configuration} \\
 \hline 
 \parbox[t]{2mm}{\multirow{3}{*}{\rotatebox[origin=c]{90}{NMF}}} 
 & ml-100K  	& num.factors=100, max.iter=10 \\[4pt]
 & ml-1M  		& num.factors=300, max.iter=10 \\[4pt]
 %& smartvote  	& ((((num.factors=100, max.iter=10)))) \\[2pt]
  \hline
  \parbox[t]{2mm}{\multirow{3}{*}{\rotatebox[origin=c]{90}{ SVD++\hspace{28pt}}}} 
 & ml-100K 	& num.factors=5, max.iter=100, \\  && learn.rate=0.01 -max -1 -bold-driver,\\ && reg.lambda=0.1 -u 0.1 -i 0.1 -b 0.1 \\ && -s 0.001 \\[2pt]
 & ml-1M 		& num.factors=10, max.iter=80, \\  && learn.rate=0.005 -max -1 -bold-driver,\\ && reg.lambda=0.05 -u 0.05 -i 0.05  \\ && -b 0.05 -s 0.001 \\[2pt]
 %& smartvote 	& (((( not crossvalidated )))) \\[2pt]
 \hline
 \parbox[t]{2mm}{\multirow{3}{*}{\rotatebox[origin=c]{90}{ItemKNN  }}} 
 & ml-100K  	& similarity=PCC, num.shrinkage=2500,\\ && num.neighbors=40 \\[4pt]
 & ml-1M  		& similarity=PCC, num.shrinkage=10, \\ && num.neighbors=80 \\[4pt]
 %& smartvote  	& ((((  )))) \\[4pt]
 \hline
 \parbox[t]{2mm}{\multirow{3}{*}{\rotatebox[origin=c]{90}{UserKNN  }}} 
 & ml-100K  	& similarity=PCC, num.shrinkage=25, \\ && num.neighbors=60 \\[4pt]
 & ml-1M  		& similarity=PCC, num.shrinkage=25, \\ && num.neighbors=80\\[4pt]
 %& smartvote  	& ((((  )))) \\[4pt]
 \hline
\parbox[t]{2mm}{\multirow{3}{*}{\rotatebox[origin=c]{90}{ NNM}}} 
 & ml-100K  	& $D=3$, max.iter $ = 16$ \\[6pt]
 & ml-1M  		& $D=8$, max.iter $ = 16$ \\[6pt]
 %& smartvote  	& $D=3$, max.iter $ = 10$ \\[2pt]
 \hline
 \parbox[t]{2mm}{\multirow{3}{*}{\rotatebox[origin=c]{90}{ Quantum}}} 
 & ml-100K  	& $D=2$, max.iter $ = 16$ \\[9pt]
 & ml-1M  		& $D=3$, max.iter $ = 16$ \\[9pt]
 %& smartvote  	& $D=3$, max.iter $ = 10$ \\[2pt]
 \hline
 \end{tabular}
 \end{table}

We note that $E_{iz}$ is invariant under the change $\vec{v}_{u_{iz}} \mapsto -\vec{v}_{u_{iz}}$. By assumption there exists a $Z$-dimensional NNM $\bigl( (\vec{p}_{u})_{u}, (\vec{E}_{iz})_{iz} \bigr)$ such that for all $i,z,z'$, $\vec{p}_{u_{iz}}^T \vec{E}_{iz'} = \delta_{zz'}$. By a sequence of arguments similar to the sequence of arguments that lead to~\eqref{fejhw4hnfn4m5}, we can show that
\beq\label{jh345m}
	\vec{E}_{iz} = \vec{p}_{u_{iz}} = \vec{e}_{n_{iz}}
\eeq
for some $n_{iz} \in [D]$. Here, $(\vec{e}_{i})_{j} = \delta_{ij}$ denotes the canonical orthonormal basis in $\AR^D$. Setting
\[
	\hat{E}_{iz} := \hat{\rho}_{u_{iz}} := \vec{e}_{n_{iz}} \vec{e}_{n_{iz}}^T
\]
we arrive at a quantum model induced by the NNM $\bigl( (\vec{p}_{u})_{u}, (\vec{E}_{iz})_{iz} \bigr)$~\footnote{Furthermore, we note that our assumptions imply that $\mathbb{P}_{u}[ \hat{E}_{i} = z ]$ is binary.}. Both $(\vec{e}_{n_{iz}})_{z}$ and $(\vec{v}_{u_{iz}})_{z}$ form orthonormal bases in $\AR^D$; recall~\eqref{fewt4eg5}. Let $U_{i}$ be the unitary transformation defined by 
\beq\label{few54ff}
	U_{i}: \, \vec{e}_{n_{iz}} \mapsto \vec{v}_{u_{iz}}
\eeq 
for all $z \in [Z]$. We claim that for all items $i,i' \in [I]$, $U_{i} = U_{i'}$. 

To prove this claim, assume $U_{i} \neq U_{i'}$ for some items $i \neq i'$. Hence, there must exist $k \in [D]$ such that 
\beq\label{dhjwhjwhr}
	U_{i} \vec{e}_{k} \neq U_{i'} \vec{e}_{k}.
\eeq
By~\eqref{jh345m} and $Z = D$ there exist $z,z' \in [Z]$ with the property that
\beq\label{fnewjk545}
	n_{iz} = k, \ \ \ n_{i'z'} = k
\eeq
so that by~\eqref{jh345m},
\beq\label{fekjwfklwejrk}
	1 = \vec{e}_{n_{iz}}^T \vec{e}_{n_{i'z'}} = \vec{E}_{iz}^T \vec{p}_{u_{i'z'}}.
\eeq
By~\eqref{few54ff} and~\eqref{dhjwhjwhr}, 
\[
	\vec{v}_{u_{iz}} = U_{i}\vec{e}_{k} \neq U_{i'}\vec{e}_{k} = \vec{v}_{u_{i'z'}}.
\]
Hence, by~\eqref{fejhw4hnfn4m5} and $\| \vec{v}_{u_{iz}} \|_{2} = 1$ for all $i,z$,
\[
	1 > ( \vec{v}_{u_{iz}}^T \vec{v}_{u_{i'z'}})^2 = \tr( E_{iz} \rho_{u_{i'z'}} ) = \vec{E}_{iz}^T \vec{p}_{u_{i'z'}}.
\]
This contradicts~\eqref{fekjwfklwejrk} and we conclude that $U_{i} = U_{i'}$ for all $i,i' \in [I]$.
 
It follows that there exists an unitary transformation $U$ with the property that 
\beq
	E_{iz} = \rho_{u_{iz}}  = U \; \vec{e}_{n_{iz}} \vec{e}_{n_{iz}}^T \; U^T
\eeq
(recall~\eqref{fejhw4hnfn4m5}, \eqref{jh345m}, and \eqref{few54ff}). Therefore, all the user and item matrices commute and we can extract the optimal NNM from the quantum model by computing the eigenvectors of $\rho_1, ..., \rho_D$.

%%%%%
\section{Configurations}

Table~\ref{wefk4354} specifies the configuration of the algorithms used in the numerical experiments aiming at small MAE and RMSE. The performance of existing algorithms was evaluated using the java library LibRec\footnote{http://www.librec.net/}.

%%%%%
%\bibliographystyle{unsrt}
%\bibliography{Recommender_systems_inspired_by_the_structure_of_quantum_theory_arXiv}  

\end{document}